\documentclass[review]{elsarticle}

\usepackage{lineno,hyperref}
\modulolinenumbers[5]

\usepackage{amsmath}
\usepackage{amssymb}
\usepackage{graphicx}
\usepackage{subcaption}
\usepackage{algorithm}
\usepackage[noend]{algpseudocode}

\usepackage{tikz}

\usepackage{amsthm}
\newtheorem{theorem}{Theorem}
\newtheorem{corollary}{Corollary}
\newtheorem{definition}{Definition}

\journal{AIJ}

\begin{document}

\begin{frontmatter}

\title{Compact and Efficient Encodings for Planning in Factored State and Action Spaces with Learned Binarized Neural Network Transition Models\tnoteref{mytitlenote}}
\tnotetext[mytitlenote]{Parts of this work appeared in preliminary form in Say and Sanner, 2018~\cite{Say2018}.}


\author{Buser Say, Scott Sanner\\
   \{bsay,ssanner\}@mie.utoronto.ca}
\address{Department of Mechanical \& Industrial Engineering, University of Toronto, 
   Canada\\ 
   Vector Institute, Canada}




\begin{abstract}
In this paper, we leverage 
the efficiency of Binarized Neural Networks (BNNs) to 
learn complex state transition models of planning domains 
with discretized factored state and action spaces.  
In order to directly exploit this transition structure for 
planning, we present two novel compilations of the learned 
factored planning problem with BNNs based on reductions to 
Weighted Partial Maximum Boolean Satisfiability (FD-SAT-Plan+) 
as well as Binary Linear Programming (FD-BLP-Plan+). 
Theoretically, we show that our SAT-based Bi-Directional Neuron 
Activation Encoding is asymptotically the most compact encoding 
relative to the current literature and supports Unit Propagation (UP) -- 
an important property that facilitates efficiency in SAT solvers.
Experimentally, we validate the computational 
efficiency of our Bi-Directional Neuron Activation Encoding 
in comparison to an existing neuron activation encoding and  
demonstrate the ability to learn complex transition 
models with BNNs.  We test the runtime efficiency of both 
FD-SAT-Plan+ and FD-BLP-Plan+ on the learned factored planning 
problem showing that FD-SAT-Plan+ scales better with increasing 
BNN size and complexity. Finally, we present a finite-time 
incremental constraint generation algorithm based on generalized 
landmark constraints to improve the planning accuracy of our 
encodings through simulated or real-world interaction.
\end{abstract}


\end{frontmatter}


\section{Introduction}

Deep neural networks (DNNs) have significantly improved the ability 
of autonomous systems to perform complex tasks, such as image 
recognition~\cite{Krizhevsky2012}, speech recognition~\cite{Deng2013} 
and natural language processing~\cite{Collobert2011}, and can outperform 
humans and human-designed super-human systems (i.e., systems that can 
achieve better performance than humans) in complex planning tasks such 
as Go~\cite{Alphago2016} and Chess~\cite{Alphazero2017}.  

In the area of learning and planning, recent work on 
HD-MILP-Plan~\cite{Say2017} has
explored a two-stage framework that (i) learns transition 
models from data with ReLU-based DNNs and (ii) plans 
optimally with respect to the learned transition models using 
Mixed-Integer Linear Programming, but did not provide encodings 
that are able to learn and plan with \emph{discrete} state variables.  
As an alternative to ReLU-based DNNs, Binarized Neural Networks 
(BNNs)~\cite{Hubara2016} have been introduced with the specific ability 
to learn compact models over discrete variables, providing a new 
formalism for transition learning and planning in 
factored~\cite{Boutilier1999} discretized state and action spaces that 
we explore in this paper. However planning with these BNN transition 
models poses two non-trivial questions: (i) What is the most efficient 
compilation of BNNs for planning in domains with factored state and 
(concurrent) action spaces? (ii) Given that BNNs may learn incorrect 
domain models and a planner can sometimes have limited access to 
real-world (or simulated) feedback, how can the planner 
repair BNN compilations to improve their planning accuracy?

To answer question (i), we present two novel compilations of the learned 
factored planning problem with BNNs based on reductions to Weighted Partial 
Maximum Boolean Satisfiability (FD-SAT-Plan+) and Binary Linear Programming 
(FD-BLP-Plan+). Theoretically, we show that the SAT-based Bi-Directional 
Neuron Activation Encoding is asymptotically the most compact encoding relative to the current literature~\cite{Boudane2018} and is efficient with Unit Propagation 
(UP). Experimentally, we demonstrate the computational efficiency of our 
Bi-Directional Neuron Activation Encoding compared to the existing neuron 
activation encoding~\cite{Say2018}. Then, we test the 
effectiveness of learning complex state transition models with BNNs, and 
test the runtime efficiency of both FD-SAT-Plan+ and FD-BLP-Plan+ on the 
learned factored planning problems over four factored planning domains 
with multiple size and horizon settings. 

While there are methods for 
learning PDDL models from data~\cite{Yang2007,Amir2008} and excellent 
PDDL planners~\cite{Helmert2006,Richter2010}, we remark that BNNs are 
strictly more expressive than PDDL-based learning paradigms for learning 
concurrent effects in factored action spaces that may depend on the joint 
execution of one or more actions. Furthermore, while Monte Carlo Tree 
Search (MCTS) methods~\cite{Kocsis2006,Keller2013} including 
AlphaGo~\cite{Alphago2016} and AlphaGoZero~\cite{Alphago2016} 
\emph{could} technically plan with a  BNN-learned black box model of 
transition dynamics, unlike this work, they do not exploit 
the BNN transition structure (i.e., they simply sample it as a black box) and they do not provide 
optimality guarantees with respect to the learned model (unless they exhaustively sample all trajectories).  Other methods that combine deep learning and planning such as ASNets~\cite{Toyer2018} focus on exploiting (P)PDDL domain structure for learning policies, but do not learn BNN transition models from raw data and then plan with these BNN models as we focus on here.

To answer question (ii), that is, for the additional scenario when the 
planner has further access to real-world (or simulated) feedback, we introduce 
a finite-time constraint generation algorithm based on generalized landmark 
constraints from the decomposition-based cost-optimal classical 
planner~\cite{Davies2015}, where we detect and constrain invalid sets of 
action selections from the decision space of the planners and efficiently 
improve their planning accuracy through simulated or real-world interaction. 

In sum, this work provides the first two planners capable of 
learning complex transition models in domains with mixed (continuous 
and discrete) factored state and action spaces as BNNs, and capable of 
exploiting their structure in Weighted Partial Maximum Boolean 
Satisfiability and Binary Linear Programming encodings for planning 
purposes. 
Theoretically, we show the compactness and efficiency of our SAT-based 
encoding, and the finite-time convergence of our incremental algorithm. Empirical 
results show the computational efficiency of our new Bi-Directional 
Neuron Activation Encoding demonstrates strong performance for 
FD-SAT-Plan+ and FD-BLP-Plan+ in both the learned and original domains, 
and provide a new transition learning and planning formalism to the 
data-driven model-based planning community.

\section{Preliminaries}

Before we present the Weighted Partial Maximum Boolean Satifiability 
(WP-MaxSAT) and Binary Linear Programming (BLP) compilations of the 
learned planning problem, we review the preliminaries motivating this 
work. We begin this section by describing the formal notation and the 
problem definition that is used in this work.


\subsection{Problem Definition}

A deterministic factored~\cite{Boutilier1999} planning problem is a tuple 
$\Pi = \langle S,A,C,T,I,G,R \rangle$ where $S=\{s_1, \dots ,s_{n_1}\}$ 
and $A=\{a_1, \dots ,a_{n_2}\}$ are sets of state and action variables 
for positive integers $n_1, n_2$ with domains $D_{s_1}, \dots, D_{s_{|S|}}$ 
and $D_{a_1}, \dots, D_{a_{|A|}}$ respectively, $C: D_{s_1} \times \dots 
\times D_{s_{|S|}} \times D_{a_1} \times \dots \times D_{a_{|A|}} \rightarrow 
\{true,false\}$ is the global function representing the properties of state 
and action variables that must be true for all time steps, $T: D_{s_1} \times 
\dots \times D_{s_{|S|}} \times D_{a_1} \times \dots \times D_{a_{|A|}} 
\rightarrow D_{s_1} \times \dots \times D_{s_{|S|}}$ denotes the stationary 
transition function between two time steps, and $R: D_{s_1} \times \dots \times 
D_{s_{|S|}} \times D_{a_1} \times \dots \times D_{a_{|A|}} \rightarrow \mathbb{R}$ 
is the reward function. Finally, $I: D_{s_1} \times \dots \times D_{s_{|S|}} 
\rightarrow \{true,false\}$ is the initial state function that defines the initial 
values of state variables, and $G: D_{s_1} \times \dots \times D_{s_{|S|}} 
\rightarrow \{true,false\}$ is the goal state function that defines the properties 
of state variables that must be true at the last time step. In this work, we assume 
the initial value $V_i \in \mathbb{R}$ of each state variable $s_i \in S$ is known.

Given a planning horizon $H$, a solution $\pi = \langle \bar{A}^1, \dots, \bar{A}^H 
\rangle$ (i.e. plan) to problem $\Pi$ is a tuple of values $\bar{A}^t = \langle \bar{a}^t_1 , 
\dots, \bar{a}^t_{|A|} \rangle \in D_{a_1} \times \dots \times D_{a_{|A|}}$ for all action 
variables $A$ and time steps $t\in \{1,\dots,H\}$ (and a tuple of values $\bar{S}^t = 
\langle \bar{s}^t_1 , \dots, \bar{s}^t_{|S|} \rangle \in D_{s_1} \times \dots \times 
D_{s_{|S|}}$ for all state variables $S$ and time steps $t\in \{1,\dots,H+1\}$) 
such that $T(\langle \bar{s}^t_1 , \dots, \bar{s}^t_{|S|}, \bar{a}^t_1 , \dots, 
\bar{a}^t_{|A|} \rangle) = \bar{S}^{t+1}$ and $C(\langle \bar{s}^t_1 , \dots, 
\bar{s}^t_{|S|}, \bar{a}^t_1 , \dots, \bar{a}^t_{|A|} \rangle) = true$ for all time 
steps $t\in \{1,\dots,H\}$, and $I(\bar{S}^{1}) = true$ for time step $t=1$ and $G(\bar{S}^{H+1}) = true$ for time step $t=H+1$, where $\bar{x}^t$ denotes the value 
of variable $x \in A \cup S$ at time step $t$. Similarly, an optimal 
solution to $\Pi$ is a plan such that the total reward $\sum_{t=1}^{H}R(\langle 
\bar{s}^{t+1}_1 , \dots, \bar{s}^{t+1}_{|S|}, \bar{a}^t_1 , \dots, \bar{a}^t_{|A|} 
\rangle)$ is maximized. For notational simplicity, we denote the tuple of variables 
$\langle x_{e_{1}}, \dots, x_{e_{|E|}} \rangle$ as $\langle x_e | e \in E \rangle$ 
given set $E$, and use the symbol ${}^\frown$ for the concatenation of two tuples. Next, we introduce an example domain with a complex transition structure.

\subsection{Example Domain: Cellda}
\label{sec:example_domain}

\begin{figure}
    \begin{subfigure}{.32\textwidth}
    \centering
        \includegraphics[width=\linewidth]{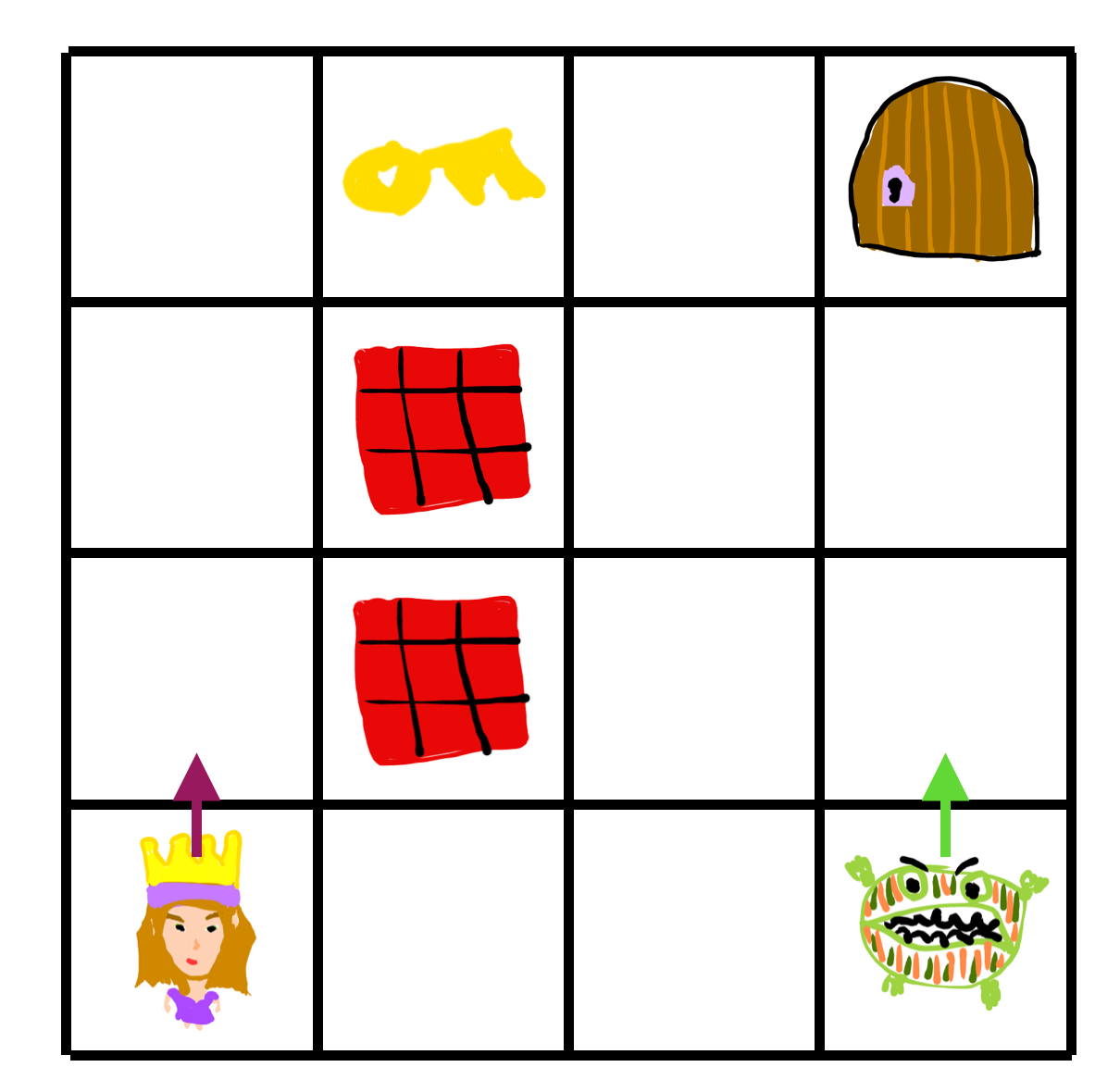}
        \caption{Time step t=1}
        \label{fig:cellda_naive_0}
    \end{subfigure}
    \begin{subfigure}{.32\textwidth}
    \centering
        \includegraphics[width=\linewidth]{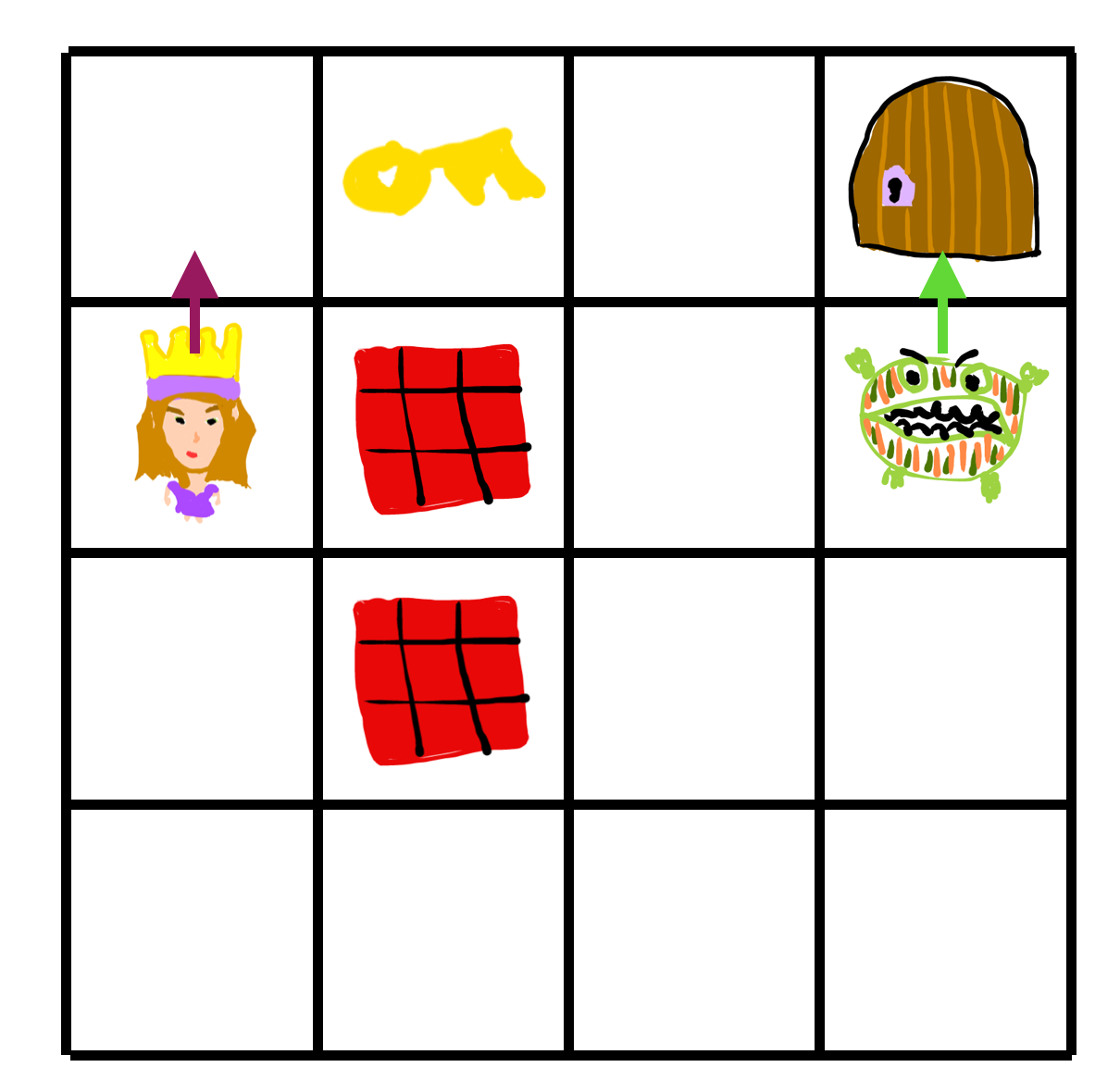}
        \caption{Time step t=3}
        \label{fig:cellda_naive_1}
    \end{subfigure}
    \begin{subfigure}{.32\textwidth}
    \centering
        \includegraphics[width=\linewidth]{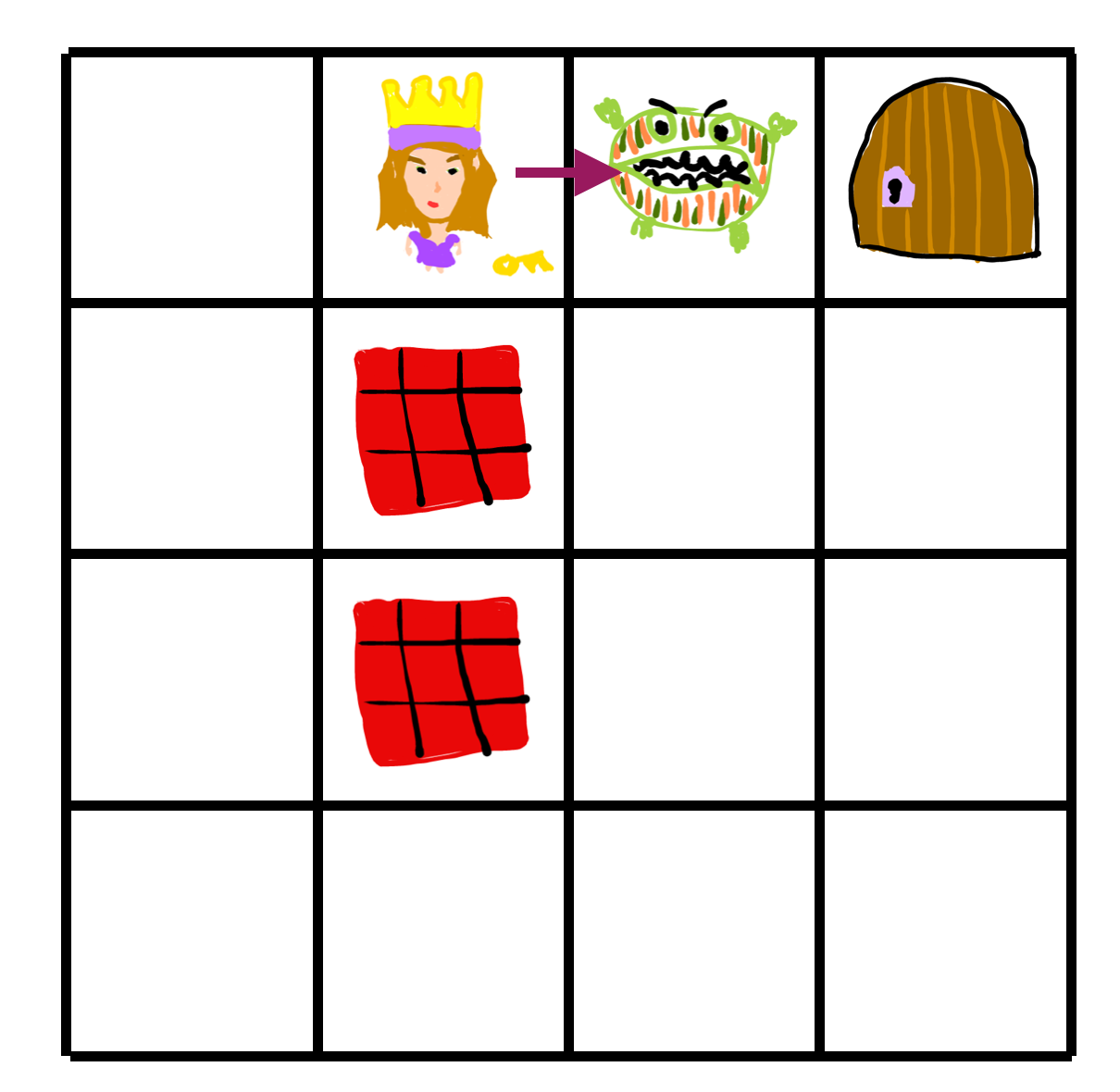}
        \caption{Time step t=5}
        \label{fig:cellda_naive_2}
    \end{subfigure}
 
    \begin{subfigure}{.32\textwidth}
    \centering
        \includegraphics[width=\linewidth]{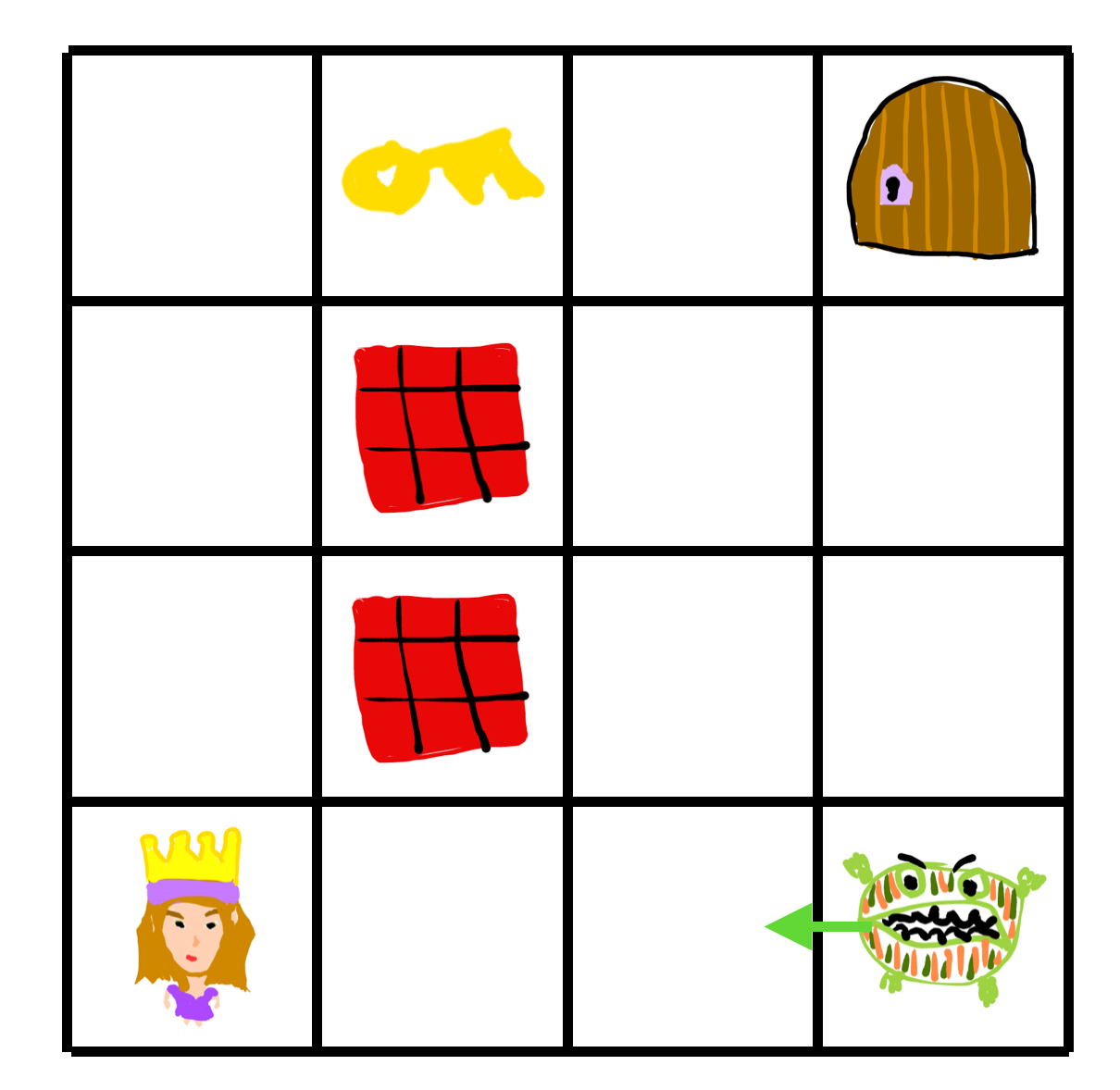}
        \caption{Time step t=1}
        \label{fig:cellda_learn_0}
    \end{subfigure}
    \begin{subfigure}{.32\textwidth}
    \centering
        \includegraphics[width=\linewidth]{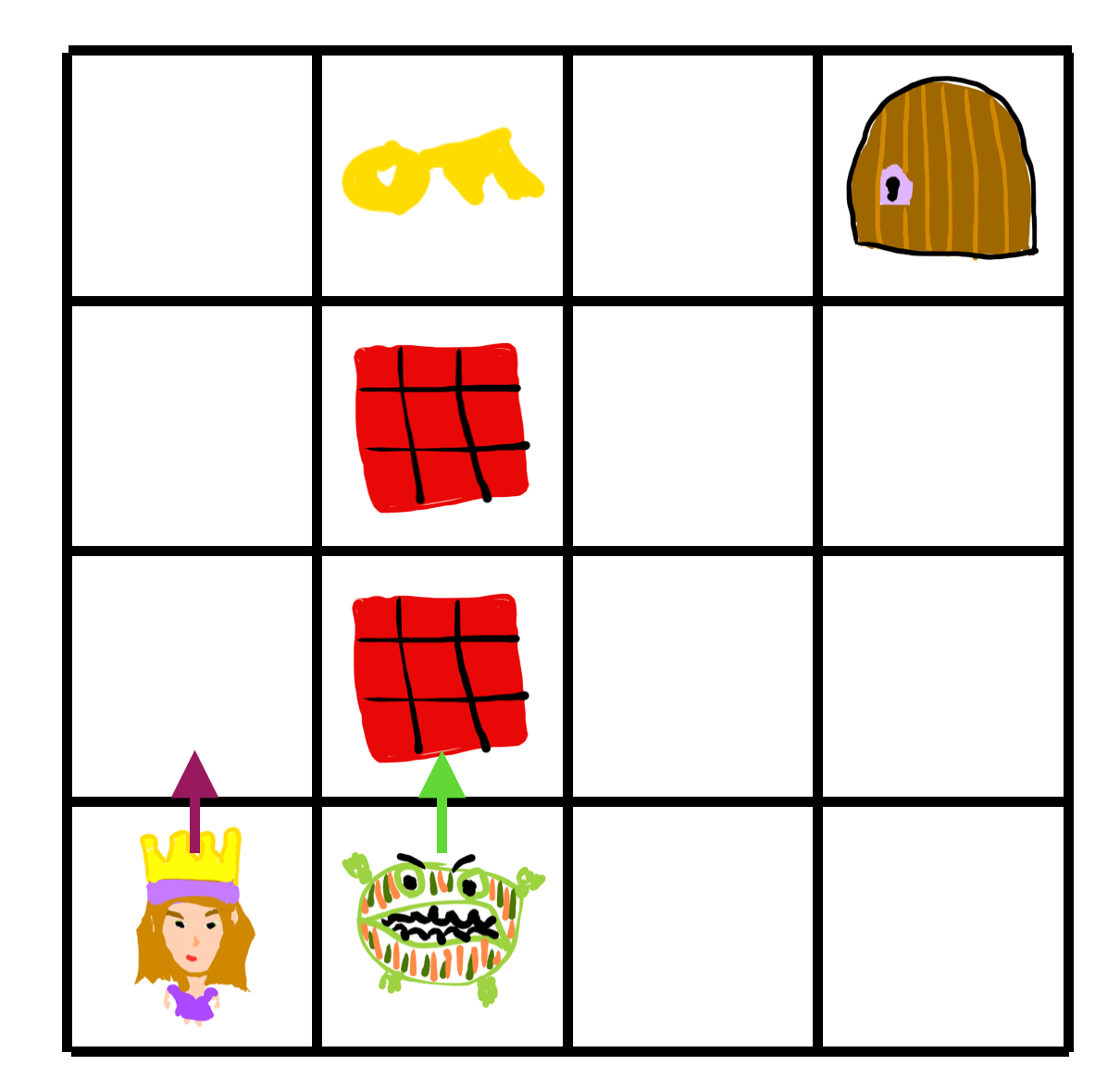}
        \caption{Time step t=3}
        \label{fig:cellda_learn_1}
    \end{subfigure}
    \begin{subfigure}{.32\textwidth}
    \centering
        \includegraphics[width=\linewidth]{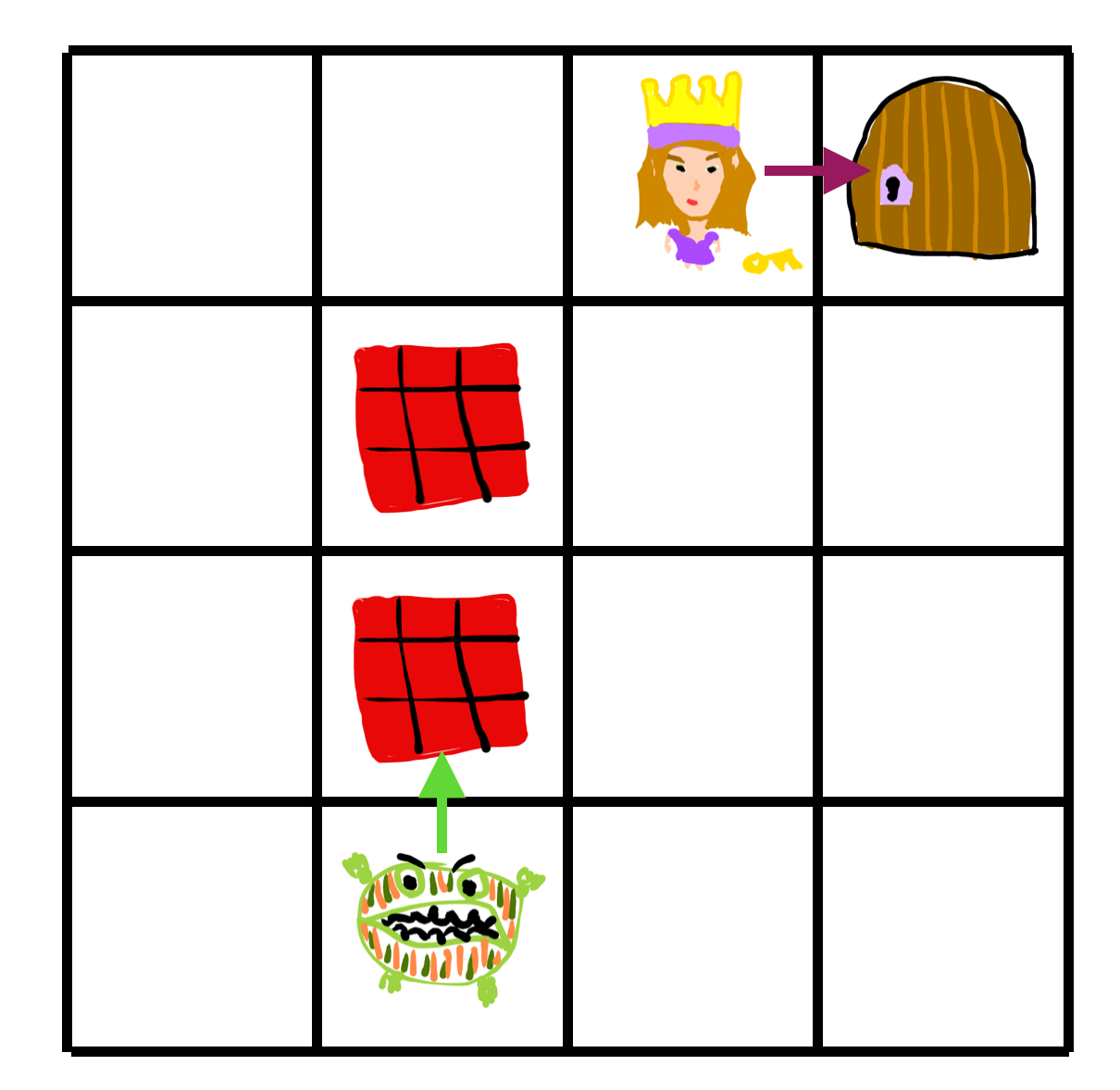}
        \caption{Time step t=9}
        \label{fig:cellda_learn_2}
    \end{subfigure}
    \caption{Visualization of two different attempts of solving 
the automated planning problem, namely Cellda. In all figures, 
the character (i.e., Cellda) that is controlled by the planner is 
represented by a humanoid character wearing a golden crown, the 
other character (i.e., the enemy) is represented by the green 
monster. The planner must control Cellda so that she obtains the 
key and escapes through the door without getting hit by the enemy.
The top three figures (\ref{fig:cellda_naive_0}-\ref{fig:cellda_naive_2}) 
represent an unsuccessful attempt of solving the problem whereas 
the bottom three figures (\ref{fig:cellda_learn_0}-\ref{fig:cellda_learn_2}) 
represent a successful attempt of solving the problem. In the unsuccessful 
attempt, Cellda moves to obtain the key by moving up three times followed 
by a move to the right. Once the key is obtained, Cellda keeps moving to 
the right but gets intercepted by the enemy. In the successful 
attempt, Cellda waits for two time steps and allows the enemy to approach 
her. Once the enemy is located underneath the block (i.e., represented by 
the red squares), Cellda starts moving up three times and to the right in 
order to first obtain the key, and then moves right towards the door to 
escape. In order to solve this complex automated planning problem, the 
planner must make long-term decisions while accounting for the 
\textit{unknown} adversarial behaviour of the enemy.}
  \label{fig:cellda_intro}
\end{figure}

Influenced by the famous video game {\it The Legend of Zelda}~\cite{Nintendo1986}, the 
Cellda domain models an agent who must escape from a two dimensional ($N$-by-$N$) 
dungeon cell through an initially locked door by obtaining a key without getting 
hit by the enemy. Next, we describe the Cellda domain as a factored planning 
problem $\Pi$ and illustrate two planning scenarios in Figure~\ref{fig:cellda_intro}.

\begin{itemize}
    \item Location of the agent, location of the enemy, whether 
the key is obtained or not and whether the agent is alive or not 
are represented by six state variables $S = \{s_1, \dots, s_6\}$ 
where state variables $s_1$ and $s_2$ represent the horizontal and 
vertical locations of the agent with integer domains, state variables 
$s_3$ and $s_4$ represent the horizontal and vertical locations of the 
enemy with integer domains, state variable $s_5$ represents whether 
the key is obtained or not with a Boolean domain, and state variable 
$s_6$ represents whether the agent is alive or not with a Boolean 
domain.
    \item Intended movement of the agent is represented by four action 
variables $A = \{a_1, a_2, a_3, a_4\}$ with Boolean domains where action 
variables $a_1$, $a_2$, $a_3$ and $a_4$ represent whether the agent 
intends to move up, down, right or left, respectively.
    \item Mutual exclusion on the intended movement of the agent, the 
boundaries of the maze and requirement that the agent must be alive are 
represented by global constraint function $C$ where 
$C(\langle s_1, \dots, s_6, a_1, \dots, a_4\rangle) = true$ 
if and only if $a_1 + a_2 + a_3 + a_4 \leq 1$, $0 \leq s_1 , s_2 < N$ and 
$s_6 = true$, $C(\langle s_1, \dots, s_6, a_1, \dots, a_4\rangle) = false$ 
otherwise.
    \item The starting location of the agent, location of the enemy, 
whether the key is currently obtained or not, and whether the agent is 
currently alive or not are represented by the initial 
state function $I$ where $I(\langle s_1, \dots, s_6\rangle) = true$ 
if and only if $s_1 = V_{s_1}$, $s_2 = V_{s_2}$, $s_3 = V_{s_3}$, 
$s_4 = V_{s_4}$, $s_5 = false$ and $s_6 = true$; 
$I(\langle s_1, \dots, s_{N^2}\rangle) = false$ otherwise where $V_{s_1}$ 
and $V_{s_2}$ denote the location of the agent, and $V_{s_3}$ 
and $V_{s_4}$ denote the location of the enemy.
    \item Final location of the agent (i.e., the location of the door), 
requirement that the agent must be alive and requirement 
that the key must be obtained are represented by the goal state 
function $G$ where $G(\langle s_1, \dots, s_6\rangle) = true$ if and only 
if $s_1 = V'_{s_1}$, $s_2 = V'_{s_2}$, $s_5 = true$ and $s_6 = true$
$G(\langle s_1, \dots, s_6\rangle) = false$ otherwise where $V'_{s_1}$ 
and $V'_{s_2}$ denote the goal location of the agent.
    \item The cost objective is to minimize total number of intended movements by 
the agent and is represented by the reward function $R$ where 
$R(\langle s_1, \dots, s_6, a_1, \dots, a_4\rangle) = a_1 + a_2 + a_3 + a_4$.
    \item Next location of the agent, next location of the enemy, 
whether the key will be obtained or not, and whether the agent will 
be alive or not are represented by the state 
transition function $T$ that is a nonlinear function of state and action 
variables $s_1, \dots, s_6, a_1, \dots, a_4$. The next location of the 
agent is a function of whether it is alive or not, its previous location 
and its intended movement. The next location of the enemy is a function 
of its previous location and the previous location of the agent and the 
intended movement of the agent. Whether the agent will be alive or not is 
a function of whether it was previously alive, and the
locations of the agent and the enemy. Finally, whether Cellda will have 
the key or not is a function of whether it previously had the key and its 
location.
\end{itemize}

Realistically, the enemy has an adversarial deterministic policy that 
is unknown to Cellda which will try to minimize the total Manhattan distance 
between itself and Cellda by breaking the symmetry first in vertical axis. 
The complete description of this domain can be found in \ref{app:dom}.
Given that the state transition function $T$ that describes the location of the enemy 
must be learned, a planner that fails to learn the adversarial policy of the enemy 
(e.g., $\pi_1$ as visualized in 
Figure (\ref{fig:cellda_naive_0}-\ref{fig:cellda_naive_2})) can get hit by the enemy.  
In contrast, a planner that learns the adversarial policy of the enemy 
(e.g., $\pi_2$ as visualized in Figure (\ref{fig:cellda_learn_0}-\ref{fig:cellda_learn_2})) 
avoids getting hit by the enemy in this scenario by waiting for two time steps to trap 
her enemy, who will try to move up for the remaining time steps and fail to intercept 
Cellda. To solve this problem, next we describe a learning and planning framework that 
(i) learns an unknown transition function $T$ from data, and (ii) plans optimally with 
respect to the learned deterministic factored planning problem.

\subsection{Factored Planning with Learned Transition Models using Deep Neural Networks and Mixed-Integer Linear Programming}

Factored planning with DNN learned transition models is a two-stage 
framework for learning and solving nonlinear factored planning 
problems as first introduced in HD-MILP-Plan~\cite{Say2017} that we briefly review now. 
Given samples of state transition data, the first stage of HD-MILP-Plan 
learns the transition function $\tilde{T}$ using a DNN with Rectified Linear Units 
(ReLUs)~\cite{Nair2010} and linear activation units. In the second stage, 
the learned transition function $\tilde{T}$ is used to construct the 
learned factored planning problem $\tilde{\Pi} = \langle S,A,C,\tilde{T},I,G,R \rangle$. 
That is, the trained DNN with fixed weights is used to predict 
the values $\bar{S}^{t+1}$ of state variables at time step $t+1$ for values $\bar{S}^{t}$ of state variables and values $\bar{A}^{t}$ of action variables at time step $t$ such that 
$\tilde{T}(\bar{S}^{t}{}^\frown \bar{A}^{t}) = \bar{S}^{t+1}$. 
As visualized in Figure~\ref{fig:hdmilpplan}, the learned transition function 
$\tilde{T}$ is sequentially chained over horizon $t\in \{1,\dots,H\}$, and 
compiled into a Mixed-Integer Linear Program yielding the planner 
HD-MILP-Plan~\cite{Say2017}. Since HD-MILP-Plan utilizes only ReLUs and linear activation 
units in its learned transition models, the state variables $s\in S$ are restricted to 
have only continuous domains $D_{s} \subseteq \mathbb{R}$.

\begin{figure}[t!]
\centering
\includegraphics[width=\linewidth]{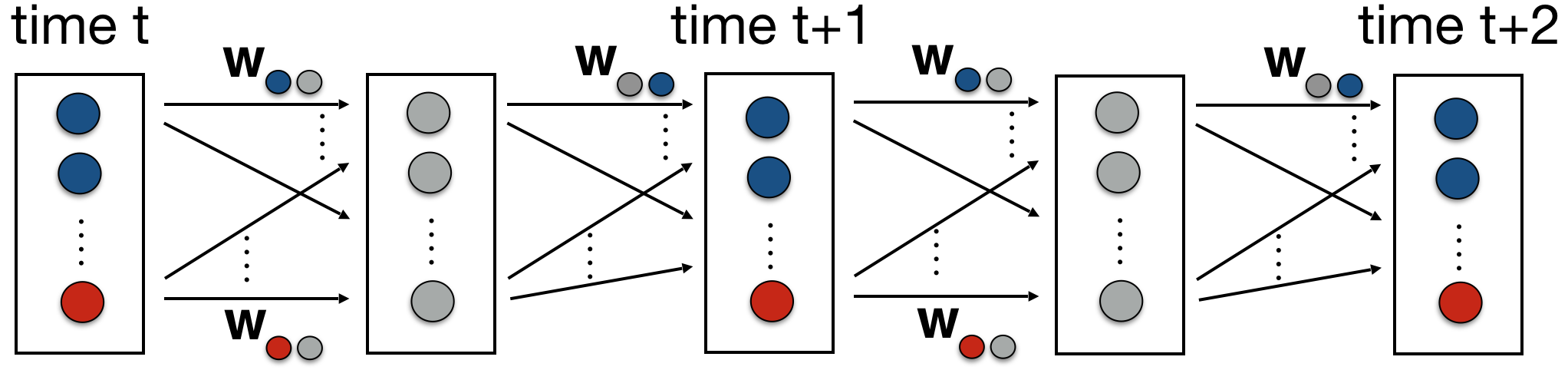}
\caption{Visualization of the learning and planning framework~\cite{Say2017}, 
where blue circles represent state 
variables $S$, red circles represent action variables $A$, gray 
circles represent hidden units (i.e., ReLUs and linear activation units for HD-MILP-Plan~\cite{Say2017}, and binary hidden units for 
FD-SAT-Plan+ and FD-BLP-Plan+) and $\textbf{w}$ represent the weights of a 
DNN. During 
the learning stage, the weights $\textbf{w}$ are learned from data. In 
the planning stage, the weights are fixed and the planner optimizes a 
given reward function with respect to the free action $A$ and 
state $S$ variables.}
\label{fig:hdmilpplan}
\end{figure}

Next, we describe an efficient DNN structure for 
learning discrete models, namely Binarized Neural Networks (BNNs)~\cite{Hubara2016}.

\subsection{Binarized Neural Networks}

Binarized Neural Networks (BNNs) are neural networks with binary weights 
and activation functions~\cite{Hubara2016}. As a result, BNNs naturally 
learn discrete models by replacing most arithmetic operations with bit-wise 
operations. Before we describe how BNN-learned transitions relate to 
HD-MILP-Plan in Figure~\ref{fig:hdmilpplan}, we first provide a technical 
description of the BNN architecture, where BNN layers are stacked in the 
following order:

\paragraph{Real or Binary Input Layer}
Binary units in all layers, with the exception of the first layer, 
receive binary input. When the inputs of the first layer are non-binary 
domains, signed binary representations up to $m$ bits of precision are 
used~\cite{Hubara2016}. For example, the integer value $\tilde{x} = -93$ can be 
represented using $m=8$ bits of precision as $\langle x^{8} = 1, x^{7} = 0, 
x^{6} = 1, x^{5} = 0, x^{4} = 0, x^{3} = 0, x^{2} = 1, x^{1} = 1\rangle$ using 
the formula $\tilde{x}=-2^{m-1}x^{m} + \sum_{i=1}^{m-1} 2^{i-1}x^{i}$. Therefore 
in the remaining of the paper, we assume the inputs of the BNNs have Boolean 
domains representing the binarized domains of the original inputs with $m$ bits 
of precision.


\paragraph{Binarization Layer}
Given input $x_{j,l}$ of binary unit 
$j\in J(l)$ at layer $l\in \{1,\dots,L\}$ the deterministic activation function 
used to compute output $y_{j,l}$ is: $y_{j,l}=1$ if $x_{j,l}\geq 0$, $-1$ otherwise, 
where $L$ denotes the number of layers and $J(l)$ denotes the set of binary 
units in layer $l\in \{1,\dots,L\}$. 

\paragraph{Batch Normalization Layer}
For all layers $l \in \{1,\dots,L\}$, Batch Normalization~\cite{Ioffe2015} is 
a method for transforming the weighted sum of outputs at layer $l-1$ in 
$\Delta_{j,l} = \sum_{i\in J(l-1)} w_{i,j,l-1} y_{i,l-1}$ to input $x_{j,l}$ 
of binary unit $j\in J(l)$ at layer $l$ using the formula: 
$x_{j,l} = \frac{\Delta_{j,l} - \mu_{j,l}} 
{\sqrt[]{\sigma^{2}_{j,l} + \epsilon_{j,l}}}\gamma_{j,l} + \beta_{j,l}$ where 
parameters $w_{i,j,l-1}$, $\mu_{j,l}$, $\sigma^{2}_{j,l}$, $\epsilon_{j,l}$, 
$\gamma_{j,l}$ and $\beta_{j,l}$ denote the weight, input mean, input variance, 
numerical stability constant (i.e., epsilon), input scaling and input bias 
respectively, where all parameters are computed at training time.

In order to place BNN-learned transition models in the same planning and learning 
framework of HD-MILP-Plan~\cite{Say2017}, we simply note that once the above BNN 
layers are learned, the Batch Normalization layers reduce to simple linear transforms 
(i.e., as we will show in Section~\ref{sec:bnn_clauses}, once parameters $w_{i,j,l-1}$, 
$\mu_{j,l}$, $\sigma^{2}_{j,l}$, $\epsilon_{j,l}$, $\gamma_{j,l}$ and $\beta_{j,l}$ 
are fixed, $x_{j,l}$ is a linear function of $y_{i,l-1}$). 
This results in a BNN with layers as visualized in Figure~\ref{fig:hdmilpplan}, 
where (i) all weights $\mathbf{w}$ are restricted to either +1 or -1 and (ii) all 
nonlinear transfer functions at BNN units are restricted to thresholded counts of 
inputs. The benefit of the BNN over the ReLU-based DNNs for 
HD-MILP-Plan~\cite{Say2017} is that it can directly model discrete variable 
transitions and BNNs can be translated to both Binary Linear Programming (BLP) and 
Weighted Partial Maximum Boolean Satisfiability (WP-MaxSAT) problems discussed next.

\subsection{Weighted Partial Maximum Boolean Satisfiability Problem}

In this work, one of the planning encodings that we focus on is
Weighted Partial Maximum Boolean Satisfiability \mbox{(WP-MaxSAT)}.   \mbox{WP-MaxSAT} is the 
problem of finding a value assignment to the variables of a Boolean formula that 
consists of hard clauses and weighted soft clauses such that:
\begin{enumerate}
    \item all hard clauses evaluate 
to true (i.e., standard SAT)~\cite{Davis1960}, and
    \item the total weight of the unsatisfied soft clauses is minimized.
\end{enumerate}

While WP-MaxSAT is known to be \textit{NP-hard}, 
state-of-the-art WP-MaxSAT solvers are experimentally shown to scale well for large 
instances~\cite{Davies2013}.

\subsection{Cardinality Networks}

When compiling BNNs to satisfiability encodings, \emph{it is critical to encode the 
counting (cardinality) threshold of the binarization layer as compactly as possible since smaller encoding sizes positively impact both compilation and optimization times}.
Cardinality Networks 
$CN_{k}^{=}(\langle x_1,\dots,x_n\rangle \rightarrow \langle c_1,\dots,c_k\rangle)$ 
provide an efficient encoding in conjunctive normal form (CNF) for counting the number 
of true assignments to Boolean variables $x_1,\dots,x_n$ using auxiliary counting 
variables $c_i$ such that $\min (\sum_{j=1}^{n}x_j, i) = \sum_{j=1}^{i} c_j$ holds 
for all $i \in \{1,\dots, k\}$ where $k$ is selected to be the smallest power of $2$ such that $k>p$~\cite{Asin2009}. As visualized in Figure~\ref{fig:CN_complete}, $CN_{k}^{=}$ is made up of three smaller building blocks, namely Half Merging (HM) Networks, Half Sorting (HS) Networks and Simplified Merging (SM) Networks, that recursively sort and merge the input Boolean variables $x_1,\dots,x_n$ with respect to their values. The detailed CNF encoding of $CN_{k}^{=}$ is outlined in 
\ref{app:cnf}.

\begin{figure}[H]
\centering
    \begin{subfigure}{.48\textwidth}
        \includegraphics[width=1\linewidth]{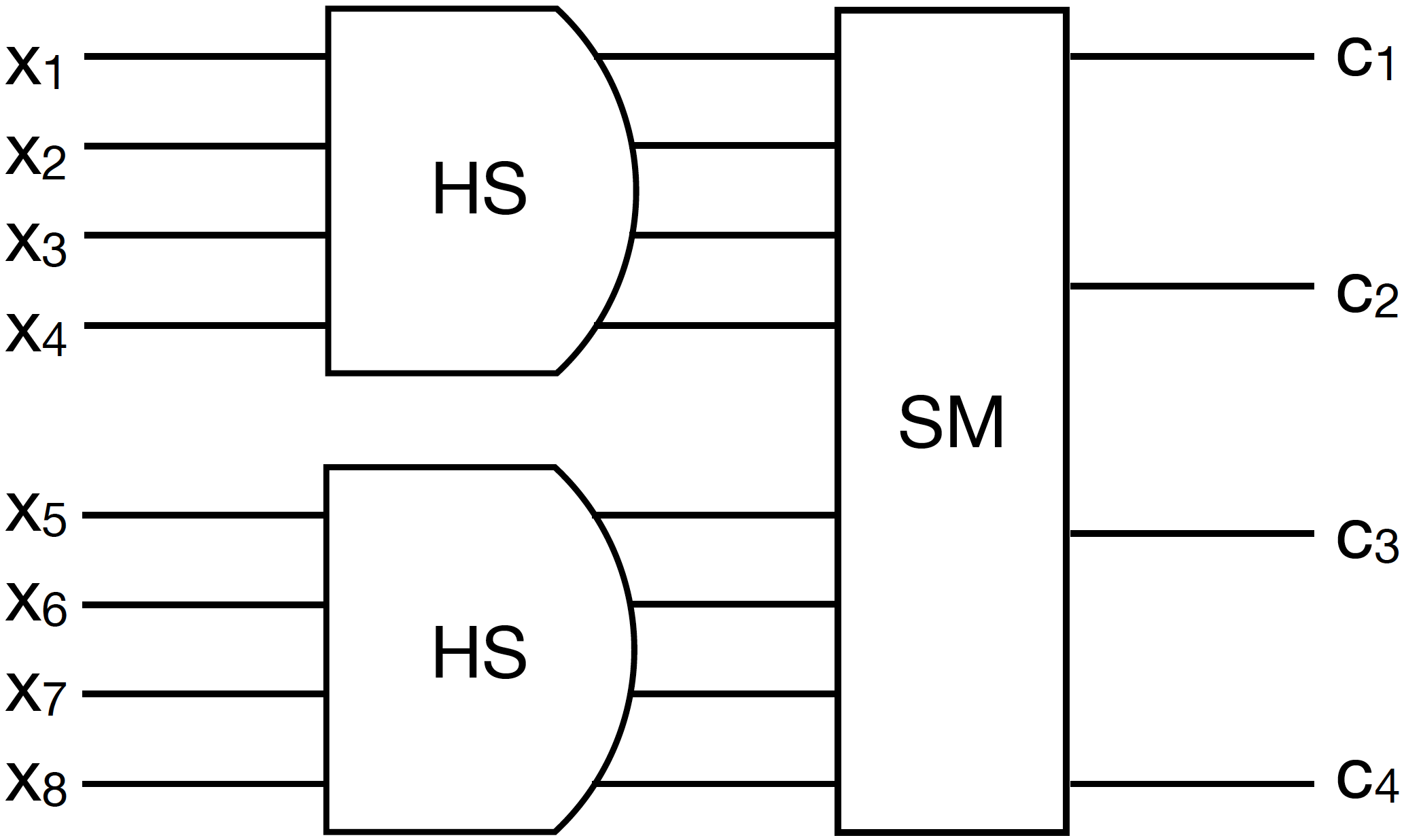}
        \caption{Cardinality Networks}
        \label{fig:CN}
    \end{subfigure}\hfill
    \begin{subfigure}{.4\textwidth}
        \includegraphics[width=1\linewidth]{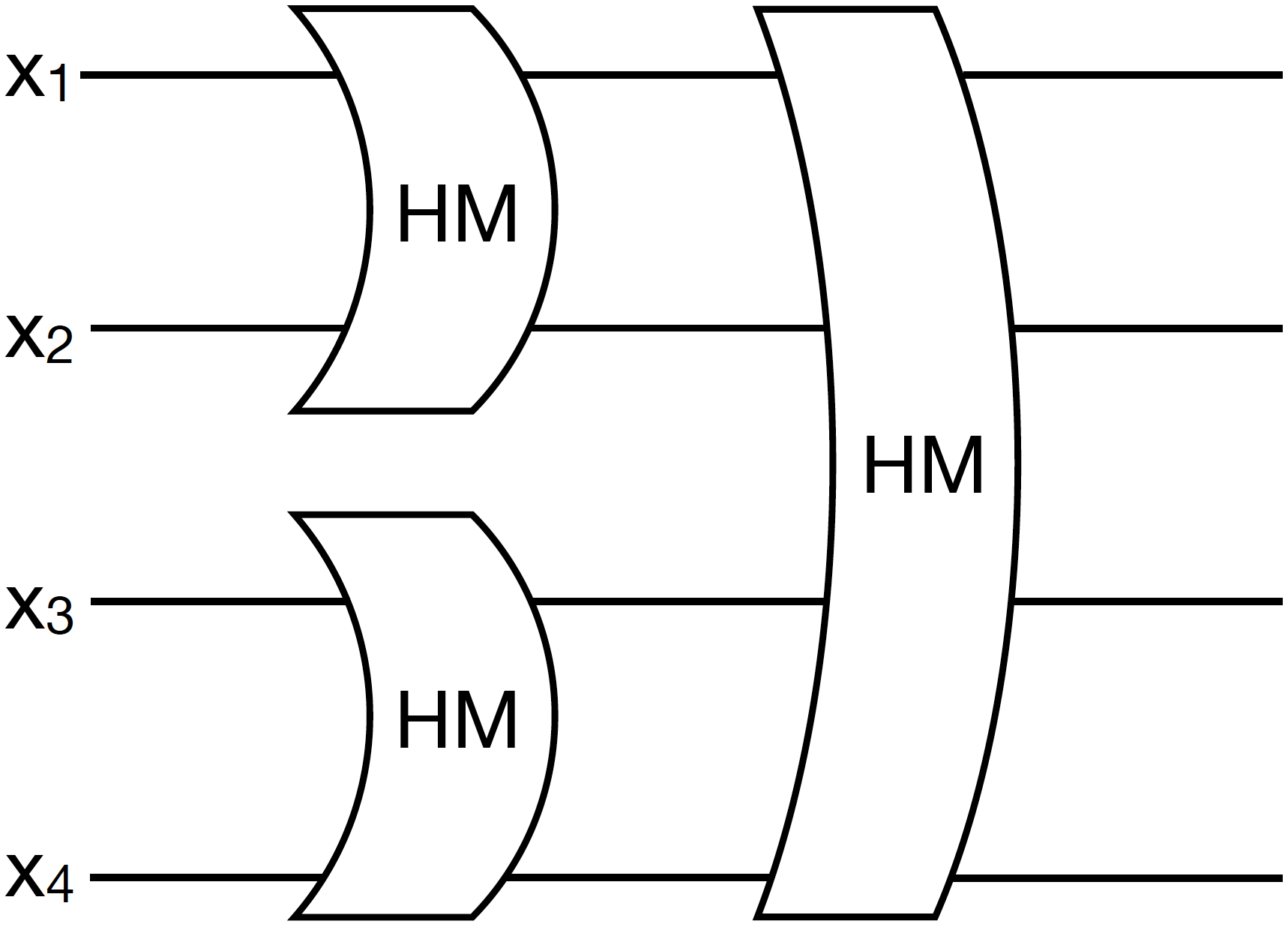}
        \caption{Half Sorting Networks}
        \label{fig:HS}
    \end{subfigure}
    \caption{Visualization of Cardinality Networks (on the left)~\cite{Asin2009} that consist of (i) Simplified Merge Networks (SM) and (ii) Half Sorting Networks (HS), and each HS further consists of Half Merging Networks (HM). An example Cardinality Network $CN_{4}^{=}(\langle x_1,\dots,x_8\rangle \rightarrow \langle c_1,\dots,c_4\rangle)$ takes in the tuple of variables $\langle x_1,\dots,x_8\rangle$ and counts upto $k=4$ variables that are assigned to true using additional auxiliary counting variables and respective hard clauses.}
    \label{fig:CN_complete}
\end{figure}

\subsection{Boolean Cardinality Constraints}

A Boolean cardinality constraint $Exactly_p(\langle x_1,\dots,x_n\rangle)$ 
describes the number of Boolean variables that are allowed to be true, and is in 
the form of $\sum_{i=1}^{n}x_i = p$. Given $CN_{k}^{=}$, 
$Exactly_p(\langle x_1,\dots,x_n\rangle)$ is defined as:
\begin{align}
Exactly_p(\langle x_1,\dots,x_n\rangle) =& \bigwedge^{r}_{i=n+1} (\neg x_i) \wedge (\neg c_{p+1}) \wedge (c_{p})\nonumber\\
 &\wedge CN_k^{=}(\langle x_1,\dots,x_{n+r}\rangle \rightarrow \langle c_1,\dots,c_{k}\rangle)\label{card1}
\end{align}
where $r$ is the smallest size of additional input variables needed to ensure 
the number of input variables is a multiple of $k$. Boolean 
cardinality constraint $Exactly_p(\langle x_1,\dots,x_n\rangle)$ is 
encoded using $O(nlog^2_2k)$ number of variables and hard clauses~\cite{Asin2009}. 

Note that the 
cardinality constraint $\sum_{i=1}^{n}x_i = p$ is equivalent to 
$\sum_{i=1}^{n} (1-x_i) = n - p$. Since Cardinality Networks 
require the value of $p$ to be less than or equal to $\frac{n}{2}$, 
Boolean cardinality constraints of the form $\sum_{i=1}^{n}x_i = p$ with 
$p>\frac{n}{2}$ must be converted into 
$Exactly_{n-p}(\langle \neg x_1,\dots,\neg x_n\rangle)$. 

Finally, a Boolean cardinality constraint encoded in CNF is said to 
be efficient if it allows efficient algorithms, such as Unit Propagation 
(UP)~\cite{Dowling1984}, to deduce the values of its variables whenever 
possible. A Boolean cardinality constraint encoded in CNF, such as 
$Exactly_{p}(\langle x_1,\dots,x_n\rangle)$, is said to be efficient 
with UP if and only if the CNF encoding allows UP to deduce:
\begin{enumerate}
    \item when exactly $p$ variables from the set $\{x_1,\dots,x_n\}$ are 
assigned to true, assignment of the remaining (unassigned) $n-p$ variables 
to false,
    \item when exactly $n-p$ variables from the set $\{x_1,\dots,x_n\}$ are 
assigned to false, assignment of the remaining (unassigned) $p$ variables 
to true,
    \item when at 
least $p + 1$ variables from the set $\{x_1,\dots,x_n\}$ are assigned to 
true, the Boolean cardinality constraint is not satisfiable, and
    \item when at 
least $n - p + 1$ variables from the set $\{x_1,\dots,x_n\}$ are assigned to 
false, the Boolean cardinality constraint is not satisfiable.
\end{enumerate}
In practice, the ability to be efficient with 
algorithms such as UP (as opposed to search) is one of 
the most important properties for the efficiency of a Boolean 
cardinality constraint encoded in 
CNF~\cite{Sinz2005,Bailleux2006,Asin2009,Jabbour2014}. 
It has been shown that $Exactly_{p}(\langle x_1,\dots,x_n\rangle)$ 
encoding is efficient with UP~\cite{Asin2009}.

\subsection{Binary Linear Programming Problem}

As an alternative to WP-MaxSAT encodings of BNN transition models, 
we can also leverage Binary Linear Programs (BLPs).
The BLP problem requires finding 
the optimal value assignment to the variables of a mathematical 
model with linear constraints, linear objective function, and binary 
decision variables. Similar to WP-MaxSAT, 
BLP is \textit{NP-hard}. The state-of-the-art 
BLP solvers~\cite{IBM2017} utilize branch-and-bound algorithms 
and can handle cardinality constraints efficiently in the size of 
its encoding.

\subsection{Generalized Landmark Constraints}

In this section, we review generalized landmark constraints that are 
necessary for improving the planning accuracy of the learned models 
through simulated or real-world interaction when the plans for the 
learned planning problem $\tilde{\Pi}$ are infeasible for the planning 
problem ${\Pi}$. A generalized landmark constraint is in the 
form of $\bigvee_{a\in A} (z_a \geq k_a)$ where the decision variable 
$z_a \in \mathbb{N}$ counts the total number of times action $a\in A$ 
is executed in the plan, and $k_a$ denotes the minimum number of times 
action $a$ must occur in a plan~\cite{Davies2015}. 
The decomposition-based planner, \textit{OpSeq}~\cite{Davies2015}, 
incrementally updates generalized landmark constraints to find 
cost-optimal plans to classical planning problems.

\section{Model Assumptions}

Before we describe our compilation-based planners for solving 
the learned planning problem $\tilde{\Pi}$, we present the set of 
assumptions used to model $\Pi$.
\begin{itemize}
    \item The deterministic 
factored planning problem $\Pi$ is static (i.e., $\Pi$ does not 
change between the time data was collected, planning, and simulation or execution).
    \item Boolean-valued functions $I$, $C$ and $G$ only take in arguments 
with Boolean domains representing the domains of state and action variables 
with $m$ bits of precision where the value of $m$ is selected prior to the 
training time of $\tilde{T}$, and is assumed to be known. Since 
functions $I$, $C$ and $G$ must always be satisfied by $\pi$, we further assume 
$I$, $C$ and $G$ can be equivalently represented by a finite set of constraints 
with $m$ bits of precision. 
Specifically, $I$ can be equivalently represented by a finite set of equality 
constraints (i.e., $I$ sets the value of every state variable $s_i \in S$ to their 
respective initial value $V_i$), and $C$ and $G$ can be represented by a finite set 
of linear constraints which are in the form of $\sum_{i=1}^{n}a_i x_i \leq p$ 
for state and action variables $x_i \in S \cup A$ where $a_i \in \mathbb{N}$ 
and $p\in \mathbb{Z}_{\geq 0}$.
    \item The reward function $R$  only takes in arguments with Boolean domains 
and is in the form of $\sum_{i=1}^{n}b_i x_i$ for state and action variables 
$x_i \in S \cup A$ with $m$ bits of precision where $a_i \in \mathbb{N}$ and 
$b_i \in \mathbb{R}_{\geq 0}$.
\end{itemize}

\section{Weighted Partial Maximum Boolean Satisfiability Compilation 
of the Learned Factored Planning Problem}

In this section, we show how to reduce the learned factored planning 
problem $\tilde{\Pi}$ with BNNs into WP-MaxSAT, which we denote as 
Factored Deep SAT Planner (FD-SAT-Plan+). FD-SAT-Plan+ uses the same 
learning and planning framework as HD-MILP-Plan~\cite{Say2017} that 
is visualized in Figure~\ref{fig:hdmilpplan} where the ReLU-based 
DNN is replaced by a BNN~\cite{Hubara2016} and the 
compilation of $\tilde{\Pi}$ is a WP-MaxSAT instead of a Mixed-Integer 
Linear Program (MILP).

\subsection{Propositional Variables}

First, we describe the set of propositional variables used in FD-SAT-Plan+.
We use three sets of propositional variables: action decision variables, 
state decision variables and BNN binary unit decision variables, where 
we use signed binary representation upto $m$ bits of precision for action 
and state variables with non-binary domains. 

\begin{itemize}
\item ${X}^{i}_{a,t}$ denotes if $i$-th bit of action variable $a\in A$ is 
executed at time step $t\in \{1,\dots,H\}$ (i.e., each bit of an action 
variable corresponds to a red circle in Figure~\ref{fig:hdmilpplan}).
\item ${Y}^{i}_{s,t}$ denotes if $i$-th bit of state variable $s\in S$ is 
true at time step $t\in \{1,\dots,H+1\}$ (i.e., each bit of a state variable corresponds to a blue circle in Figure~\ref{fig:hdmilpplan}).
\item ${Z}_{j,l,t}$ denotes if BNN binary unit $j\in J(l)$ at layer 
$l\in \{1,\dots,L\}$ is activated at time step $t\in \{1,\dots,H\}$ (i.e., each BNN binary unit corresponds to a gray circle in Figure~\ref{fig:hdmilpplan}).
\end{itemize}

\subsection{Constants and Indexing Functions}

Next we define the additional constants and indexing functions used in FD-SAT-Plan+. 
\begin{itemize}
\item $V^{i}_{s}$ is the initial (i.e., at $t=1$) value of the $i$-th 
bit of state variable $s\in S$.
\item $In(x,i)$ is the function that maps the $i$-th bit of a state or 
an action variable $x\in S \cup A$ to the corresponding binary unit in 
the input layer of the BNN such that $In(x,i)=j$ where $j\in J(1)$.
\item $Out(s,i)$ is the function that maps the $i$-th bit of a state 
variable $s\in S$ to the corresponding binary unit in the output layer 
of the BNN such that $Out(s,i)=j$ where $j\in J(L)$.
\end{itemize}

\subsection{The WP-MaxSAT Compilation}

Below, we define the WP-MaxSAT encoding of the learned factored planning 
problem $\tilde{\Pi}$ with BNNs. First, we present the hard clauses 
(i.e., clauses that must be satisfied) used in FD-SAT-Plan+.

\subsubsection{Initial State Clauses}
The following conjunction of hard clauses encodes the initial state 
function $I$.
\begin{align}
&\bigwedge_{s \in S} \bigwedge_{1 \leq i \leq m} ({Y}^{i}_{s,1} 
\leftrightarrow V^{i}_{s})\label{sat1}
\end{align}
where hard clause (\ref{sat1}) sets the initial values of the state 
variables at time step $t=1$.

\subsubsection{Bi-Directional Neuron Activation Encoding}

In this section, we present a CNF encoding to model the activation 
behaviour of a BNN binary unit $j\in J(l), l\in \{1,\dots,L\}$ that 
requires only $O(nlog^2_2k)$ variables and hard clauses, and is 
efficient with Unit Propagation (UP) where $k$ is selected to be the 
smallest power of $2$ such that $k>p$. 

Given the input variables $x_1, \dots, x_n$ and the binary activation 
function with the activation threshold $p$, the output of a binary 
unit can be efficiently encoded in CNF as follows. First, the Boolean 
variable $v$ is defined to represent the activation of a binary unit such 
that $v=true$ if and only if the binary unit is activated, and $v=false$ 
otherwise. Then, Cardinality Networks $CN_{k}^{=}$ are used to count the 
input variables $x_1, \dots, x_n$, and an additional bi-directional relation 
(i.e., $(v \leftrightarrow c_{p})$) is used to relate the counting variable 
$c_{p}$ to the output variable $v$ as follows.
\begin{align}
    &Act_{p}(v,\langle x_1, \dots, x_n\rangle) = \bigwedge^{r}_{i=n+1} (\neg x_i) \wedge (v \vee \neg c_{p}) \wedge (\neg v \vee c_{p})\nonumber\\
 &\wedge CN_k^{=}(\langle x_1,\dots,x_{n+r}\rangle \rightarrow \langle c_1,\dots,c_{k}\rangle)\label{act1}
\end{align}
In hard clause (\ref{act1}), we remind that the purpose of the additional 
input variables and the respective unit clauses is to ensure the number of 
input variables is a multiple of $k$ such that $r$ denotes the smallest size 
of the additional input variables needed. Note that 
$Act_{p}(v,\langle x_1, \dots, x_n\rangle)$ simply replaces hard clauses 
$(c_{p})$ and $(\neg c_{p+1})$ with hard clauses $(v \vee \neg c_{p})$ and 
$(\neg v \vee c_{p})$ from $Exactly_{p}(\langle x_1, \dots, x_n\rangle)$ to 
model the bi-directional relation $(v \leftrightarrow c_{p})$ (i.e., 
a neuron is activated $v=true$ if and only if the value of at least $p$ input 
variables is true $c_p = true$). 

As a result, unlike the previous work that uses $O(np)$ number of 
variables and hard clauses for encoding neuron activations in CNF, 
the encoding we present here uses only $O(nlog^2_2k)$ number of 
variables and hard clauses. For notational clarity, we will refer 
to the neuron activation encoding that is presented in this section 
as the Bi-Directional Neuron Activation Encoding, and the previous 
neuron activation encoding~\cite{Say2018} as the Uni-Directional 
Neuron Activation Encoding\footnote{The names of the encodings are 
selected to reflect the fact that the Uni-Directional encoding 
seperately encodes uni-directional the constraints 
$v \rightarrow (\sum_{i=1}^{n} x_i \geq p)$ and 
$v \leftarrow (\sum_{i=1}^{n} x_i \geq p)$ while the Bi-Directional 
Neuron Activation Encoding compactly encodes the bi-directional 
constraint $v \leftrightarrow (\sum_{i=1}^{n} x_i \geq p)$}.

Finally, the Uni-Directional Neuron Activation Encoding has been shown 
not to be efficient with UP~\cite{Boudane2018} in contrast to the 
Bi-Directional Neuron Activation Encoding which we show is efficient with 
UP in Theorem~\ref{proof:gac}. Given the Bi-Directional Neuron Activation 
Encoding, next we present the CNF clauses that model learned BNN transition 
models.

\subsubsection{BNN Clauses}
\label{sec:bnn_clauses}

Given the efficient CNF encoding $Act_{p}(v,\langle x_1, \dots, x_n\rangle)$, 
we present the conjunction of hard clauses to model the complete BNN model.
\begin{align}
&\bigwedge_{1\leq t\leq H} \bigwedge_{s \in S} 
\bigwedge_{1 \leq i \leq m} ({Y}^{i}_{s,t} \leftrightarrow {Z}_{In(s,i),1,t})\label{sat2}\\
&\bigwedge_{1\leq t\leq H} \bigwedge_{a \in A} 
\bigwedge_{1 \leq i \leq m} ({X}^{i}_{a,t} \leftrightarrow {Z}_{In(a,i),1,t})\label{sat3}\\
&\bigwedge_{1\leq t\leq H} \bigwedge_{s \in S} 
\bigwedge_{1 \leq i \leq m} ({Y}^{i}_{s,t+1} \leftrightarrow {Z}_{Out(s,i),L,t})\label{sat4}\\
&\bigwedge_{1\leq t\leq H} \bigwedge_{2\leq l\leq L} 
\bigwedge_{\substack{j \in J(l), \\ p^{*}_{j} \leq \lceil \frac{|J(l-1)|}{2} \rceil }} Act_{p_j}({Z}_{j,l,t}, \langle {Z}_{i,l-1,t} | i\in J(l-1), w_{i,j,l-1} = 1 \rangle \nonumber\\
&\quad {}^\frown \langle \neg {Z}_{i,l-1,t} | i\in J(l-1), w_{i,j,l-1} = -1 \rangle)\label{sat5}\\
&\bigwedge_{1\leq t\leq H} \bigwedge_{2\leq l\leq L} 
\bigwedge_{\substack{j \in J(l), \\ p^{*}_{j} > \lceil \frac{|J(l-1)|}{2} \rceil }} Act_{p_j}(\neg {Z}_{j,l,t}, \langle {Z}_{i,l-1,t} | i\in J(l-1), w_{i,j,l-1} = -1 \rangle \nonumber\\
&\quad {}^\frown \langle \neg {Z}_{i,l-1,t} | i\in J(l-1), w_{i,j,l-1} = 1 \rangle)\label{sat6}
\end{align}
where activation constant $p_{j}$ in hard clauses (\ref{sat5}-\ref{sat6}) 
are computed using the batch normalization parameters for binary 
unit $j\in J(l)$ in layer $l\in \{2,\dots,L\}$ at training time 
such that:
\begin{align}
&p^{*}_{j} = \Biggl\lceil \frac{|J(l-1)| + \mu_{j,l} - \frac{\beta_{j,l} \sqrt[]{\sigma^{2}_{j,l} + \epsilon_{j,l}}}{\gamma_{j,l}}}{2} \Biggr\rceil \nonumber\\
&\text{if }p^{*}_{j} \leq \lceil \frac{|J(l-1)|}{2} \rceil \text{ then } p_{j} = p^{*}_{j}\nonumber\\
&\text{else }p_{j} = |J(l-1)|-p^{*}_{j}+1\nonumber
\end{align}
where $|x|$ denotes the size of set $x$. The computation of the activation constant 
$p_{j}, j\in J(l)$ ensures that $p_{j}$ is less than or equal to the half 
size of the previous layer $|J(l-1)|$, as Bi-Directional Neuron Activation 
Encoding only counts upto $\lceil \frac{|J(l-1)|}{2} \rceil$.

Hard clauses (\ref{sat2}-\ref{sat3}) map the binary units at the input layer of 
the BNN (i.e., $l=1$) to a unique state or action variable, respectively. Similarly, 
hard clause (\ref{sat4}) maps the binary units at the output layer of the BNN 
(i.e., $l=L$) to a unique state variable. Note that because each state and action 
variable uniquely maps onto an input and/or an output BNN binary unit, constraint 
functions $I$, $C$ and $G$ limit the feasible set of values input and output units 
of the BNN can take. Hard clauses (\ref{sat5}-\ref{sat6}) encode the binary activation 
of every unit $j\in J(l)$ in the BNN given its activation constant $p_j$ and weights 
$w_{i,j,l-1}$ such that ${Z}_{j,l,t} \leftrightarrow (\sum_{\substack{i\in J(l-1),\\ 
w_{i,j,l-1}=1}} {Z}_{i,l-1,t} + \sum_{\substack{i\in J(l-1),\\ w_{i,j,l-1}=-1}} 
(1-{Z}_{i,l-1,t}) \geq p_j)$.

\subsubsection{Global Constraint Clauses}
The following conjunction of hard clauses encodes the global function $C$.
\begin{align}
&\bigwedge_{1\leq t\leq H} C(\langle {Y}^{i}_{s,t} | s \in S, 1 \leq i \leq m\rangle{}^\frown \langle {X}^{i}_{a,t} | a \in A, 1 \leq i \leq m\rangle)\label{sat7}
\end{align}
where hard clause (\ref{sat7}) represents domain-dependent global function 
on state and action variables. Some common examples of function $C$, such as exactly one Boolean action or state variable 
must be true, are respectively encoded by hard clauses (\ref{sat8}-\ref{sat9}) as follows. 
\begin{align}
&\bigwedge_{1\leq t\leq H} Exactly_{1}(\langle {X}^{1}_{a,t} | a\in A\rangle)\label{sat8}\\
&\bigwedge_{1\leq t\leq H} Exactly_{1}(\langle {Y}^{1}_{s,t} | s\in S\rangle)\label{sat9}
\end{align}

In general, linear constraints in the form of 
$\sum_{i=1}^{n} a_i x_i \leq p$, such as bounds on state and action variables, 
can be encoded in CNF where $a_i$ are positive integer coefficients 
and $x_i$ are decision variables with non-negative integer 
domains~\cite{Abio2014}.

\subsubsection{Goal State Clauses}
The following conjunction of hard clauses encodes the goal state function $G$.
\begin{align}
&G(\langle {Y}^{i}_{s,H+1} | s \in S, 1 \leq i \leq m \rangle)\label{sat10}
\end{align}
where hard clause (\ref{sat10}) sets the goal constraints on the state 
variables $S$ at time step $t=H+1$.

\subsubsection{Reward Clauses}

Given the reward function $R$ for each time step $t$ is in the form of $\sum_{s\in S}\sum_{i=1}^{m} f^{i}_{s} {Y}^{i}_{s,t+1} + \sum_{a\in A}\sum_{i=1}^{m} g^{i}_{a} {X}^{i}_{a,t}$, the following weighted soft clauses (i.e., optional 
weighted clauses that may or may not be satisfied where each weight corresponds to 
the penalty of not satisfying a clause):
\begin{align}
& \bigwedge_{1\leq t\leq H} \bigwedge_{1 \leq i \leq m} ( \bigwedge_{s\in S} (f^{i}_{s} , {Y}^{i}_{s,t+1}) \wedge \bigwedge_{a\in A} (g^{i}_{a} , {X}^{i}_{a,t}))\label{sat11}
\end{align}
can be written to represent $R$ where $f^{i}_{s}, g^{i}_{a} \in \mathbb{R}_{\geq 0}$ are the weights of the soft clauses for each bit of state and action variables, respectively.

\section{Binary Linear Programming Compilation of the Learned Factored Planning Problem}

Given FD-SAT-Plan+, we present the Binary Linear Programming (BLP) compilation of the 
learned factored planning problem $\tilde{\Pi}$ with BNNs, which we 
denote as Factored Deep BLP Planner (FD-BLP-Plan+).

\subsection{Binary Variables and Parameters}

FD-BLP-Plan+ uses the same set of decision variables and parameters 
as FD-SAT-Plan+.

\subsection{The BLP Compilation}

FD-BLP-Plan+ replaces hard clauses (\ref{sat1}) and (\ref{sat2}-\ref{sat4}) with equivalent linear 
constraints as follows.
\begin{align}
&{Y}^{i}_{s,1} = V^{i}_{s} &\forall{s \in S, 1 \leq i \leq m}\label{blp1}\\
&{Y}^{i}_{s,t} = {Z}_{In(s,i),1,t} &\forall{1\leq t\leq H, s \in S, 1 \leq i \leq m}\label{blp2}\\
&{X}^{i}_{a,t} = {Z}_{In(a,i),1,t} &\forall{1\leq t\leq H, a \in A, 1 \leq i \leq m}\label{blp3}\\
&{Y}^{i}_{s,t+1} = {Z}_{Out(s,i),L,t} &\forall{1\leq t\leq H, s \in S, 1 \leq i \leq m}\label{blp4}
\end{align}
Given the activation constant $p^{*}_{j}$ of binary unit $j\in J(l)$ in layer 
$l\in \{2,\dots,L\}$, FD-BLP-Plan+ replaces hard clauses (\ref{sat5}-\ref{sat6}) 
representing the activation of binary unit $j$ with the following linear 
constraints:
\begin{align}
&p^{*}_{j} {Z}_{j,l,t} \leq \sum_{\substack{i\in J(l-1), \\ w_{i,j,l-1} = 1}}{Z}_{i,l-1,t} + \sum_{\substack{i\in J(l-1), \\ w_{i,j,l-1} = -1}}(1-{Z}_{i,l-1,t})\nonumber\\
&\quad \forall{1\leq t\leq H, 2\leq l\leq L, j\in J(l)}\label{blp5}\\
&p'_{j}(1-{Z}_{j,l,t}) \leq \sum_{\substack{i\in J(l-1), \\ w_{i,j,l-1} = -1}}{Z}_{i,l-1,t}  + \sum_{\substack{i\in J(l-1), \\ w_{i,j,l-1} = 1}}(1-{Z}_{i,l-1,t})\nonumber\\
&\quad \forall{1\leq t\leq H, 2\leq l\leq L, j\in J(l)}\label{blp6}
\end{align}
where $p'_{j} = |J(l-1)|-p^{*}_{j}+1$.

Global constraint hard clauses (\ref{sat7}) and goal state hard clauses (\ref{sat10}) are 
compiled into linear constraints given they are in the form of 
$\sum_{i=1}^{n} a_i x_i \leq p$. Finally, the reward function $R$ 
with linear expressions is maximized over time steps $1\leq t\leq H$ such that:
\begin{align}
&\max \; \sum_{t = 1}^{H} \sum_{i = 1}^{m} ( \sum_{s\in S} f^{i}_{s}  {Y}^{i}_{s,t+1} + \sum_{a\in A} g^{i}_{a} {X}^{i}_{a,t})\label{blp7}
\end{align}


\section{Incremental Factored Planning Algorithm for FD-SAT-Plan+ and FD-BLP-Plan+}

In this section, we extend the capabilities of our data-driven planners 
and place them in a planning scenario where the planners have access 
to limited (and potentially expensive) simulated or real-world interaction. 
Given that the plans found for the learned factored planning problem $\tilde{\Pi}$ 
by FD-SAT-Plan+ and FD-BLP-Plan+ can be infeasible to the factored planning problem 
$\Pi$, we introduce an incremental algorithm for finding plans for $\Pi$ by 
iteratively excluding invalid plans from the search space of our planners. That is, 
we add hard clauses or constraints to the encodings of our planners to exclude 
infeasible plans from their respective search space.
Similar to \textit{OpSeq}~\cite{Davies2015}, FD-SAT-Plan+ and FD-BLP-Plan+ 
are updated with the following generalized landmark hard clauses or constraints:
\begin{align}
&\bigvee_{1\leq t \leq H}\bigvee_{a\in A}(\bigvee_{\substack{1 \leq i \leq m \\ \langle t,a,i \rangle\in \mathcal{L}_n}}(\neg {X}^{i}_{a,t})  \vee \bigvee_{\substack{1 \leq i \leq m \\ \langle t,a,i \rangle\not \in \mathcal{L}_n}} ({X}^{i}_{a,t})) \label{gen1}\\
&\sum_{t = 1}^{H} \sum_{a\in A} (\sum^{m}_{\substack{i = 1, \\ \langle t,a,i \rangle\in \mathcal{L}_n}}(1-{X}^{i}_{a,t}) 
+ \sum^{m}_{\substack{i = 1, \\ \langle t,a,i \rangle\not \in \mathcal{L}_n}} {X}^{i}_{a,t}) \geq 1 \label{gen2}
\end{align}
respectively, where $\mathcal{L}_n$ is the set of $1 \leq i \leq m$ bits of action variables 
$a\in A$ executed at time steps $1\leq t \leq H$ (i.e., $\bar{X}^{i}_{a,t} = 1$) at the 
$n$-th iteration of Algorithm~\ref{incremental} that is outlined below.

\begin{algorithm}
\footnotesize
\caption{Incremental Factored Planning Algorithm}\label{incremental}
\begin{algorithmic}[1]
\State $n = 1$, planner $=$ FD-SAT-Plan+ or FD-BLP-Plan+
\State $\mathcal{L}_n \gets \text{Solve } \tilde{\Pi} \text{ using the planner.}$
\If {$\mathcal{L}_n$ is empty (i.e., infeasibility) or $\mathcal{L}_n$ is a plan for $\Pi$}
\State \Return $\mathcal{L}_n$
\Else 
\If {\;$\text{planner}$ = FD-SAT-Plan+ }
\State {\;$\text{planner} \gets$ hard clause (\ref{gen1})}
\Else
\State {\;$\text{planner} \gets$ constraint (\ref{gen2})}
\EndIf
\EndIf
\State $n \gets n + 1$, go to line 2.
\end{algorithmic}
\end{algorithm}

For a given horizon $H$, Algorithm~\ref{incremental} iteratively computes a set of 
executed action variables $\mathcal{L}_n$, or returns infeasibility for the learned 
factored planning problem $\tilde{\Pi}$. If the set of executed action variables 
$\mathcal{L}_n$ is non-empty, we evaluate  
whether $\mathcal{L}_n$ is a valid plan for the original factored planning 
problem $\Pi$ (i.e., line 3) either in the actual domain or 
using a high fidelity domain simulator -- in our case RDDLsim~\cite{Sanner2010}. 
If the set of executed action variables $\mathcal{L}_n$ constitutes 
a plan for $\Pi$, Algorithm~\ref{incremental} returns $\mathcal{L}_n$ as a plan.  
Otherwise, the planner is updated with the new set of generalized 
landmarks to exclude $\mathcal{L}_n$ and the loop repeats. Since the 
original action space is discretized and represented upto $m$ bits of 
precision, Algorithm~\ref{incremental} can be shown 
to terminate in no more than $n=2^{|A|\cdot m\cdot H}$ iterations 
by constructing an inductive proof similar to the termination 
criteria of \textit{OpSeq}.
The outline of the proof can be found in \ref{app:proof}.
Next, we present the theoretical analysis of Bi-Directional Neuron Activation 
Encoding.

\section{Theoretical Results}

We now present our theoretical results on Bi-Directional Neuron Activation 
Encoding, and prove that Bi-Directional Neuron Activation Encoding is efficient 
with Unit Propagation (UP), which is considered to be one of the most important 
theoretical properties that facilitate the efficiency of a Boolean cardinality 
constraint encoded in CNF~\cite{Sinz2005,Bailleux2006,Asin2009,Jabbour2014}.

\begin{definition}[Unit Propagation Efficiency of Neuron Activation Encoding]
\label{def:gac}
A neuron activation encoding $v \leftrightarrow (\sum_{i=1}^{n} x_i \geq p)$ 
is efficient with Unit Propagation (UP) if and only if UP is sufficient 
to deduce the following:
\begin{enumerate}
\item For any set $X' \subset \{x_1, \dots, x_n\}$ with size $|X'|=n-p$,
value assignment to variables $v=true$, and $x_i=false$ for all $x_i\in X'$,  
the remaining $p$ variables from the set $\{x_1, \dots, x_n\} \backslash X'$ 
are assigned to true, 
\item For any set $X' \subset \{x_1, \dots, x_n\}$ with size 
$|X'|=p-1$, value assignment to variables $v=false$, and $x_i=true$ for all 
$x_i\in X'$, the remaining $n-p+1$ variables from the set 
$\{x_1, \dots, x_n\} \backslash X'$ are assigned to false, 
\item Partial value assignment of $p$ variables from $\{x_1, \dots, x_n\}$ to true 
assigns variable $v=true$, and 
\item Partial value assignment of $n-p+1$ variables from 
$\{x_1, \dots, x_n\}$ to false assigns variable $v=false$
\end{enumerate}
where $|x|$ denotes the size of set $x$.
\end{definition}

\begin{theorem}[Unit Propagation Efficiency of $Act_{p}(v,\langle x_1, \dots, x_n\rangle)$]
\label{proof:gac}
Bi-Directional Neuron Activation Encoding $Act_{p}(v,\langle x_1, \dots, x_n\rangle)$ is efficient with Unit Propagation.
\end{theorem}

\begin{proof}
To show $Act_{p}(v,\langle x_1, \dots, x_n \rangle)$ is efficient with Unit Propagation (UP), we need to show 
it exhaustively for all four cases of Definition~\ref{def:gac}.

Case 1 ($\forall{X' \subset \{x_1, \dots, x_n\}}$ where $|X'|=n-p$, $v = true$ and $x_i=false$ $\forall{x_i\in X'} \rightarrow x_i=true \; \forall{x_i\in \{x_1, \dots, x_n\} \backslash X'}$ by UP): When $v = true$, UP assigns $c_p = true$ using the hard clause $(\neg v \vee c_p)$. Given value assignment $x_i=false$ to variables $x_i\in X'$ for any set $X' \subset \{x_1, \dots, x_n\}$ with size $|X'|=n-p$, it has been shown that UP will set the remaining $p$ variables from the set $\{x_1, \dots, x_n\}\backslash X'$ to true using the conjunction of hard clauses that encode $CN^{=}_{k}$~\cite{Asin2009}.

Case 2 ($\forall{X' \subset \{x_1, \dots, x_n\}}$ where $|X'|=p-1$, $v = false$ and $x_i=true$ $\forall{x_i\in X'} \rightarrow x_i=false \; \forall{x_i\in \{x_1, \dots, x_n\} \backslash X'}$ by UP): When $v = false$, UP assigns $c_p = false$ using the hard clause $(v \vee \neg c_p)$. Given value assignment $x_i=true$ to variables $x_i\in X'$ for any set $X' \subset \{x_1, \dots, x_n\}$ with size $|X'|=p-1$, it has been shown that UP will set the remaining $n-p+1$ variables from the set $\{x_1, \dots, x_n\} \backslash X'$ to false using the conjunction of hard clauses that encode $CN^{=}_{k}$~\cite{Asin2009}.

Cases 3 ($\forall{X' \subset \{x_1, \dots, x_n\}}$ where $|X'|=p$, $x_i=true$ $\forall{x_i\in X'} \rightarrow v = true$ by UP) When $p$ variables from the set $\{x_1, \dots, x_n\}$ are set to true, it has been shown that UP assigns the counting variable $c_p = true$ using the conjunction of hard clauses that encode $CN^{=}_{k}$~\cite{Asin2009}. Given the assignment $c_p = true$, UP assigns $v = true$ using the hard clause $(v \vee \neg c_p)$. 

Cases 4 ($\forall{X' \subset \{x_1, \dots, x_n\}}$ where $|X'|=n-p+1$, $x_i=false$ $\forall{x_i\in X'} \rightarrow v = false$ by UP) When $n-p+1$ variables from the set $\{x_1, \dots, x_n\}$ are set to false, it has been shown that UP assigns the counting variable $c_p = false$ using the conjunction of hard clauses that encode $CN^{=}_{k}$~\cite{Asin2009}. Given the assignment $c_{p} = false$, UP assigns $v = false$ using the hard clause $(\neg v \vee c_p)$.
\end{proof}

We now discuss the importance of our theoretical result in the context of 
both related work and the contributions of our paper. Amongst the state-of-the-art 
CNF encodings~\cite{Boudane2018} that are efficient with UP for constraint 
$v \rightarrow (\sum_{i=1}^{n} x_i \geq p)$, 
Bi-Directional Neuron Activation Encoding uses the smallest number of variables and 
hard clauses. The previous state-of-the-art CNF encoding for constraint 
$v \rightarrow (\sum_{i=1}^{n} x_i \geq p)$ is an extension of the Sorting 
Networks~\cite{Een06} and uses $O(nlog^2_2n)$ number of variables and hard 
clauses~\cite{Boudane2018}. In contrast, Bi-Directional Neuron Activation Encoding 
is an extension of the Cardinality Networks~\cite{Asin2009}, and only uses 
$O(nlog^2_2k)$ number of variables and hard clauses, and is efficient with UP as 
per Theorem~\ref{proof:gac}.

\section{Experimental Results}

In this section, we evaluate the effectiveness of factored planning 
with BNNs. First, we present the benchmark domains used to test 
the efficiency of our learning and factored planning framework with 
BNNs. Second, we present the accuracy of BNNs to learn complex state 
transition models for factored planning problems. Third, we compare 
the runtime efficiency of Bi-Directional Neuron Activation Encoding 
against the existing Uni-Directional Neuron Activation 
Encoding~\cite{Say2018}. Fourth, we test the efficiency and scalability 
of planning with FD-SAT-Plan+ and FD-BLP-Plan+ on the learned factored 
planning problems $\tilde{\Pi}$ across multiple problem sizes and horizon 
settings. Finally, we demonstrate the effectiveness of 
Algorithm~\ref{incremental} to find a plan for the factored planning 
problem $\Pi$.

\subsection{Domain Descriptions}

The description of four automated planning problems that are designed to 
test the planning performance of FD-BLP-Plan+ and FD-SAT-Plan+ with nonlinear 
state transition functions, and state and action variables with discrete 
domains, namely Navigation~\cite{Sanner2011}, Inventory Control~\cite{Mann2014}, 
System Administrator~\cite{Guestrin2001,Sanner2011}, and 
Cellda~\cite{Nintendo1986}, follows.

\paragraph{Navigation:} The Navigation~\cite{Sanner2011} task for an 
agent in a two-dimensional ($N$-by-$N$ where $N\in \mathbb{Z}^{+}$) maze 
is cast as an automated planning problem as follows.

\begin{itemize}
    \item Location of the agent is represented by $N^2$ state variables 
$S = \{s_1, \dots, s_{N^2}\}$ with Boolean domains where state variable 
$s_i$ represents whether the agent is located at position $i$ or not.
    \item Intended movement of the agent is represented by four action 
variables $A = \{a_1, a_2, a_3, a_4\}$ with Boolean domains where action 
variables $a_1$, $a_2$, $a_3$ and $a_4$ represent whether the agent 
intends to move up, down, right or left, respectively.
    \item Mutual exclusion on the intended movement of the agent is 
represented by global function $C$ where 
$C(\langle s_1, \dots, s_{N^2}, a_1, \dots, a_4\rangle) = true$ 
if and only if $a_1 + a_2 + a_3 + a_4 \leq 1$, 
$C(\langle s_1, \dots, s_{N^2}, a_1, \dots, a_4\rangle) = false$ otherwise.
    \item Current location of the agent is represented by the initial 
state function $I$ where $I(\langle s_1, \dots, s_{N^2}\rangle) = true$ 
if and only if $s_i = V_i$ for all positions $i\in \{1,\dots,{N^2}\}$, 
$I(\langle s_1, \dots, s_{N^2}\rangle) = false$ otherwise.
    \item Final location of the agent is represented by the goal state 
function $G$ where $G(\langle s_1, \dots, s_{N^2}\rangle) = true$ if and 
only if $s_i = V'_i$ for all positions $i\in \{1,\dots,{N^2}\}$, 
$G(\langle s_1, \dots, s_{N^2}\rangle) = false$ otherwise where $V'_i$ 
denotes the goal position of the agent (i.e., $V'_i = true$ if and only if 
position $i\in \{1,\dots,{N^2}\}$ is the final location, $V'_i = false$ 
otherwise).
    \item Objective is to minimize total number of intended movements by 
the agent and is represented by the reward function $R$ where 
$R(\langle s_1, \dots, s_{N^2}, a_1, \dots, a_4\rangle) = a_1 + a_2 + a_3 + a_4$.
    \item Next location of the agent is represented by the state 
transition function $T$ that is a nonlinear function of state and action 
variables $s_1, \dots, s_{N^2}, a_1, \dots, a_4$. For each position 
$i\in \{1,\dots,{N^2}\}$, next location of the agent is defined by the 
function $T_{i}(\langle s_1, \dots, s_{N^2}, a_1, \dots, a_4\rangle) =$ 
if $r(i,j,k) \wedge s_j \wedge a_k$ then $true$, otherwise $false$ where 
$r(i,j,k)$ denotes whether position $i$ can be reached from 
position $j\in \{1,\dots,{N^2}\}$ by intended movement $k\in \{1,2,3,4\}$.
\end{itemize}

We report the results on maze sizes $N=3,4,5$ over 
planning horizons $H=4,\dots,10$. Note that this automated planning 
problem is a deterministic version of its original from 
IPPC2011~\cite{Sanner2011}.

\paragraph{Inventory Control:} The Inventory Control~\cite{Mann2014} is 
the task of managing inventory of a product with demand cycle length $N\in \mathbb{Z}^{+}$, 
and is cast as an automated planning problem as follows.

\begin{itemize}
    \item Inventory level of the product, phase of the demand cycle and 
whether inventory demand level is met or not are represented by three state 
variables $S = \{s_1, s_2, s_3\}$ where state variables $s_1, s_2$ have 
non-negative integer domains and state variable $s_3$ has a Boolean domain, 
respectively.
    \item Ordering fixed amount of inventory is represented by an action 
variable $A = \{a_1\}$ with a Boolean domain.
    \item Meeting the inventory demand level is represented by 
global function $C$ where $C(\langle s_1, s_2, s_3, a_1\rangle) = true$ 
if and only if $s_3 = true$, $C(\langle s_1, s_2, s_3, a_1\rangle) = false$ 
otherwise.
    \item Current inventory level, current step of the 
demand cycle and meeting the current inventory demand level are 
represented by the initial state function $I$ where 
$I(\langle s_1, s_2, s_3\rangle) = true$ if and only if $s_1 = V$, 
$s_2 = 0$ and $s_3 = true$, $I(\langle s_1, s_2, s_3\rangle) = false$ 
otherwise where $V$ denotes the current inventory level.
    \item  Meeting the final inventory demand level is 
represented by goal state function $G$ where 
$G(\langle s_1, s_2, s_3\rangle) = true$ if and only if 
$s_3 = true$, $G(\langle s_1, s_2, s_3\rangle) = false$ otherwise.
    \item Objective is to minimize total inventory storage 
cost and is represented by the reward function $R$ where 
$R(\langle s_1, s_2, s_3, a_1 \rangle) = c s_1$ and $c$ denotes the 
unit storage cost of inventory.
    \item Next inventory level, next step of the demand cycle and 
whether the next inventory demand level is met or not are represented 
by the state transition function $T$ that is a nonlinear 
function of state and action variables $s_1, s_2, s_3, a_1$. The next 
inventory level is defined by the function 
$T_{1}(\langle s_1, s_2, s_3, a_1\rangle) = \max(r a_1 + s_1 - d(s_2), 0)$ 
where $r$ and $d(i)$ are the fixed order amount and the demand at the $i$-th 
step of the demand cycle, respectively. The next step of the 
demand cycle is defined by the function 
$T_{2}(\langle s_1, s_2, s_3, a_1\rangle) =$ if $s_2 < N$ then $s_2 + 1$, 
otherwise $0$. Finally, whether the next inventory demand level is met or 
not is defined by the function $T_{3}(\langle s_1, s_2, s_3, a_1\rangle) =$ 
if $r a_1 + s_1 - d(s_2) \leq d^{min}$ then $false$, 
otherwise $true$ where $d^{min}$ is the minimum allowable unmet demand.
\end{itemize}

We report the results on Inventory Control tasks with two 
demand cycle lengths $N \in \{2,4\}$ over planning horizons 
$H=5,\dots,8$.

\paragraph{System Administrator:} The System 
Administrator~\cite{Guestrin2001,Sanner2011} is the maintenance task of 
a computer network of size $N$ and is cast as an automated planning 
problem as follows.

\begin{itemize}
    \item The age of computer $i\in \{1,\dots,N\}$, and whether 
computer $i\in \{1,\dots,N\}$ is running or not, are represented 
by $2N$ state variables $S = \{s_1, \dots, s_{2N}\}$ with 
non-negative integer domains and Boolean domains, respectively.
    \item Rebooting computers $i\in \{1,\dots,N\}$ are represented 
by $N$ action variables $A = \{a_1, \dots, a_{N}\}$ with Boolean 
domains.
    \item Bound on the number of computers that can be rebooted 
and the requirement that all computers must be running are 
represented by global function $C$ where 
$C(\langle s_1, \dots, s_{2N}, a_1, \dots, a_{N}\rangle) = true$ 
if and only if $\sum_{i=1}^{N} a_{i} \leq a^{max}$ and 
$s_{N+1}, \dots, s_{2N} = true$, 
$C(\langle s_1, \dots, s_{2N}, a_1, \dots, a_{N}\rangle) = false$ 
otherwise where $a^{max}$ is the bound on the number of computers 
that can be rebooted at a time.
    \item Current age of computer $i\in \{1,\dots,N\}$, and 
whether computer $i\in \{1,\dots,N\}$ is currently running 
or not are represented by the initial state function $I$ where 
$I(\langle s_1, \dots, s_{2N}\rangle) = true$ if and only if 
$s_{1}, \dots, s_{N} = 0$ and $s_{N+1}, \dots, s_{2N} = true$, 
$I(\langle s_1, \dots, s_{2N}\rangle) = false$ otherwise.
    \item Final requirement that all computers must be 
running is represented by the goal state function $G$ where 
$G(\langle s_1, \dots, s_{2N}\rangle) = true$ if and only if 
$s_{N+1}, \dots, s_{2N} = true$, 
$G(\langle s_1, \dots, s_{2N}\rangle) = false$ otherwise.
    \item Objective is to minimize total number of computer 
reboots and is represented by the reward function $R$ where 
$R(\langle s_1, \dots, s_{2N}, a_1, \dots, a_{N}\rangle) = 
\sum_{i=1}^{N} a_{i}$.
    \item Next age of computer $i\in \{1,\dots,N\}$ and whether 
computer $i\in \{1,\dots,N\}$ will be running or not, are represented 
by the state transition function $T$ that is a nonlinear function 
of state and action variables $s_1, \dots, s_{2N}, a_1, \dots, a_{N}$. 
For each computer $i\in \{1,\dots,{N}\}$, next age of computer is 
defined by the function 
$T_{i}(\langle s_1, \dots, s_{2N}, a_1, \dots, a_{N}\rangle) =$ 
if $\neg s_{i+N} \vee a_i$ then $0$, otherwise $s_i + 1$. For each 
computer $i\in \{1,\dots,{N}\}$, whether the computer will be running 
or not is defined by the function 
$T_{i+N}(\langle s_1, \dots, s_{2N}, a_1, \dots, a_{N}\rangle) =$ 
if $a_{i} \vee (s_{i+N} \wedge s_{i} \leq s^{max}) \vee (s_{i+N} \wedge 
s_{i} \cdot \left( 1-\frac{\sum_{j=1}^{N} c(i,j)s_{j+N} }{1 + \sum_{j=1}^{N} 
c(i,j)} \right) \leq d^{max})$ then $true$, otherwise $false$ where 
$c(i,j)$ denotes whether computers $i$ and $j\in \{1,\dots,{N}\}$ are 
connected or not, $d^{max}$ denotes the network density threshold, and 
$s^{max}$ is the maximum computer age.
\end{itemize}

We report the results on System Administrator tasks with 
$N \in \{4,5\}$ computers over planning horizons $H=2,3,4$.

\paragraph{Cellda:} As previously described in 
Section~\ref{sec:example_domain}, 
the agent Cellda must escape a dungeon through an initially locked door 
by obtaining its key without getting hit by her enemy. The gridworld-like 
dungeon is made up of two types of cells: (i) regular cells on which 
Cellda and her enemy can move from/to deterministically up, down, right, 
left, or wait on, and (ii) blocks that neither Cellda nor her enemy can 
stand on.

We report the results on maze size $N=4$ over planning horizons 
$H=8,\dots,12$ with two different enemy policies.

\subsection{Transition Learning Performance}

In Table~\ref{tab:testerror}, we present test errors 
for different configurations of the BNNs on each domain 
instance where the sample data was generated from the 
RDDL-based domain simulator RDDLsim~\cite{Sanner2010} 
using the code available for stochastic 
exploration policy with concurrent actions. For each 
instance of a domain, total of 200,000 state transition 
samples were collected and the data was treated as 
independent and identically distributed. After random 
permutation, the data was split into training and test 
sets with 9:1 ratio. The BNNs with the feed-forward 
structure described in Section 2.4 were trained on 
MacBookPro with 2.8 GHz Intel Core i7 16 GB memory using 
the code available~\cite{Hubara2016}. For each instance 
of a domain, the smallest BNN size (i.e., the BNN with 
the least number of neurons) that achieved less than a 
preselected test error threshold (i.e., 3\% 
test error) was chosen using a grid-search over 
preselected network structure hyperparameters, namely 
width (i.e., 36, 96, 128) and depth (i.e., 1,2,3). 
The selected BNN structure for each instance is detailed 
in Table~\ref{tab:testerror}.
Overall, Navigation instances required the smallest BNN 
structures for learning due to their purely Boolean 
state and action spaces, while both Inventory, SysAdmin 
and Cellda instances required larger BNN structures 
for accurate learning, owing to their non-Boolean state 
spaces.

\begin{table}[H]
  \centering
  \caption{Transition Learning Performance Table 
measured by error on test data (in \%) for all 
domains and instances.}
  \label{tab:testerror}
  \begin{tabular}{| l | c | c |}
    \hline
    Domain & Network Structure & Test Error (\%) \\
    \hline
    Navigation(3)& 13:36:36:9 & 0.0 \\ \hline
    Navigation(4)& 20:96:96:16 & 0.0 \\ \hline
    Navigation(5)& 29:128:128:25 & 0.0 \\ \hline
    Inventory(2)& 7:96:96:5 & 0.018 \\ \hline
    Inventory(4)& 8:128:128:5 & 0.34 \\ \hline
    SysAdmin(4)& 16:128:128:12 & 2.965 \\ \hline
    SysAdmin(5)& 20:128:128:128:15 & 0.984 \\ \hline
    Cellda(x)& 12:128:128:4 & 0.645 \\ \hline
    Cellda(y)& 12:128:128:4 & 0.65 \\ \hline
  \end{tabular}
\end{table}

\subsection{Planning Performance on the Learned Factored Planning Problems}

In this section, we present the results of two computational comparisons. 
First, we test the efficiency of Bi-Directional Neuron Activation 
Encoding to the existing Uni-Directional Neuron Activation 
Encoding~\cite{Say2018} to 
select the best WP-MaxSAT-based encoding for FD-SAT-Plan+. Second, we 
compare the effectiveness of using the selected WP-MaxSAT-based encoding 
against a BLP-based encoding, namely FD-SAT-Plan+ and 
FD-BLP-Plan+, to find plans for the learned factored planning 
problem $\tilde{\Pi}$. We ran the experiments on a MacBookPro 
with 2.8 GHz Intel Core i7 16GB memory. For FD-SAT-Plan+ and 
FD-BLP-Plan+, we used MaxHS~\cite{Davies2013} with underlying 
LP-solver CPLEX 12.7.1~\cite{IBM2017}, and CPLEX 12.7.1 
solvers respectively, with 1 hour total time limit per 
domain instance.

\subsubsection{Comparison of neuron activation encodings}

\begin{figure}[H]
    \includegraphics[width=1\linewidth]{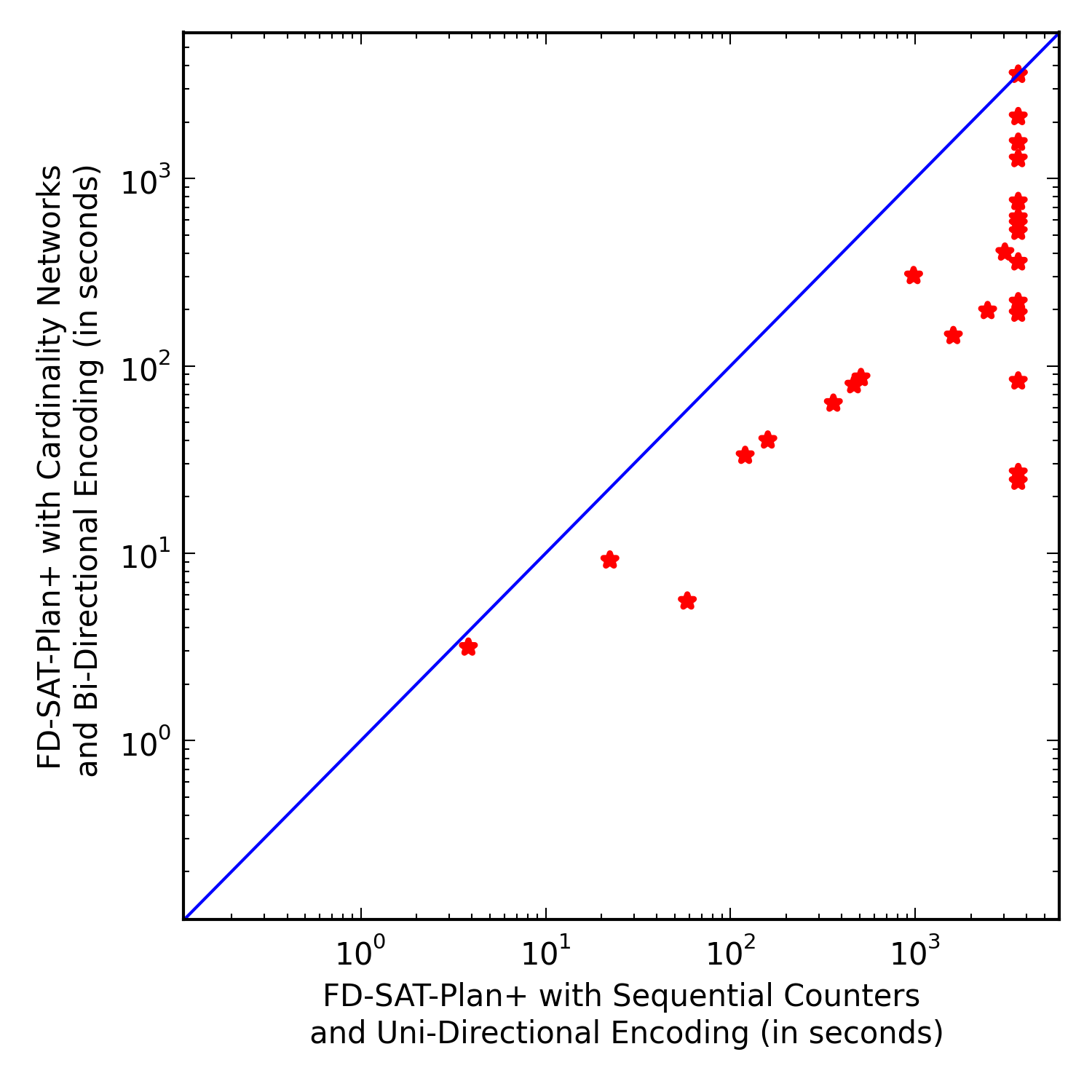}
    \caption{Timing comparison between for FD-SAT-Plan+ with Sequential Counters~\cite{Sinz2005} and Uni-Directional Encoding~\cite{Say2018} (x-axis) and Cardinality Networks~\cite{Asin2009} and Bi-Directional Encoding (y-axis). Over all problem settings, FD-SAT-Plan+ with Cardinality Networks and Bi-Directional Encoding signigicantly outperformed FD-SAT-Plan+ with Sequential Counters and Uni-Directional Encoding on all problem instances due to its (i) smaller encoding size, and (ii) UP efficiency property.}
  \label{fig:Bi_vs_Uni}
\end{figure}

The runtime efficiency of both neuron activation encodings are tested 
for the learned factored planning problems over 27 problem instances where 
we test our Bi-Directional Neuron Activation Encoding that utilizes Cardinality 
Networks~\cite{Asin2009} against the previous Uni-Directional 
Neuron Activation Encoding~\cite{Say2018} that uses Sequential Counters~\cite{Sinz2005}. 

Figure~\ref{fig:Bi_vs_Uni} visualizes the runtime comparison of both 
neuron activation encodings. The inspection of Figure~\ref{fig:Bi_vs_Uni} 
clearly demonstrate that FD-SAT-Plan+ with Bi-Directional Neuron Activation 
Encoding signigicantly outperforms FD-SAT-Plan+ with Uni-Directional Neuron 
Activation Encoding in all problem instances due to its (i) smaller encoding 
size (i.e., $O(nlog^2_2k)$ versus $O(np)$) with respect to both the number 
of variables and the number of hard clauses, and (ii) UP efficiency property 
as per Theorem~\ref{proof:gac}. Therefore, we use 
FD-SAT-Plan+ with Bi-Directional Neuron Activation Encoding in the 
remaining experiments and omit the results for FD-SAT-Plan+ with 
Uni-Directional Neuron Activation Encoding.

\subsubsection{Comparison of FD-SAT-Plan+ and FD-BLP-Plan+}

Next, we test the runtime efficiency of FD-SAT-Plan+ and FD-BLP-Plan+ for solving the learned factored planning problem.

\begin{table}[H]
  \centering
  \caption{Summary of the computational results presented in \ref{app:results} including the average 
  runtimes in seconds for both FD-SAT-Plan+ and FD-BLP-Plan+ over all four domains 
  for the learned factored planning problem within 1 hour time limit.}
  \label{tab:compresults_summary1}
  \begin{tabular}{| l | c | c |}
    \hline
    Domains & FD-SAT-Plan+ & FD-BLP-Plan+ \\ \hline
    Navigation & 529.11 & 1282.82 \\ \hline
    Inventory & 54.88 & 0.54 \\ \hline
    SysAdmin & 1627.35 & 3006.27 \\ \hline
    Cellda & 344.03 & 285.45 \\ \hline
    \hline
    Coverage & 27/27 & 20/27 \\ \hline
    Optimality Proved & 25/27 & 19/27 \\ \hline
  \end{tabular}
\end{table}

In Table~\ref{tab:compresults_summary1}, we present the summary of the 
computational results including the average runtimes in seconds, the 
total number of instances for which a feasible solution is returned 
(i.e., coverage), and the total number of instances for which an optimal 
solution is returned (i.e., optimality proved), for both FD-SAT-Plan+ 
and FD-BLP-Plan+ over all four domains for the learned factored planning 
problem within 1 hour time limit. The analysis of 
Table~\ref{tab:compresults_summary1} shows that FD-SAT-Plan+ covers all
problem instances by returning an incumbent solution to the learned 
factored planning problem compared to FD-BLP-Plan+ which runs out of 1 
hour time limit in 7 out of 27 instances before finding an incumbent 
solution. Similarly, FD-SAT-Plan+ proves the optimality of the solutions 
found in 25 out of 27 problem instances compared to FD-BLP-Plan+ which 
only proves the optimality of 19 out of 27 solutions within 1 hour time 
limit.

\begin{figure}[H]
    \includegraphics[width=1\linewidth]{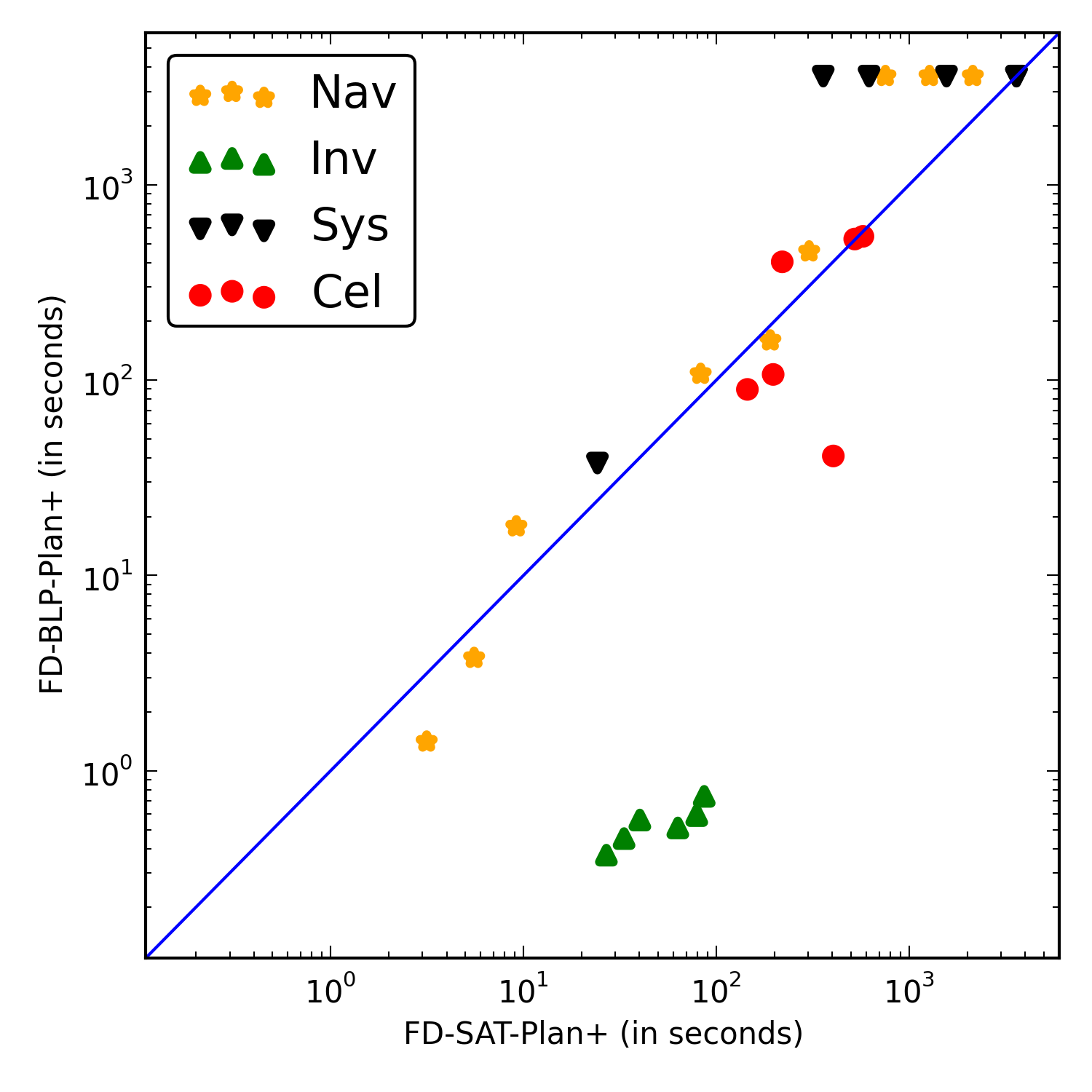}
    \caption{Timing comparison between FD-SAT-Plan+ (x-axis) and 
FD-BLP-Plan+ (y-axis). Over all problem settings, FD-BLP-Plan+ 
performed better on instances that require less than approximately 
100 seconds to solve (i.e., computationally easy instances) whereas 
FD-SAT-Plan+ outperformed FD-BLP-Plan+ on instsances that require more 
than approximately 100 seconds to solve (i.e., computationally hard 
instances).}
  \label{fig:Runtime}
\end{figure}

In Figure~\ref{fig:Runtime}, we compare the runtime performances of 
FD-SAT-Plan+ (x-axis) and FD-BLP-Plan+ (y-axis) per instance labeled 
by their domain. The analysis of Figure~\ref{fig:Runtime} across all 27 
intances shows that FD-BLP-Plan+ proved the optimality of problem 
instances from domains which require less computational demand 
(e.g., Inventory) more efficiently compared to 
FD-SAT-Plan+. In contrast, FD-SAT-Plan+ proved the optimality of problem 
instances from domains which require more computational demand (e.g., 
SysAdmin) more efficiently compared to FD-BLP-Plan+. As the instances got 
harder to solve, FD-BLP-Plan+ timed-out more compared to FD-SAT-Plan+, 
mainly due to its inability to find incumbent solutions as evident from 
Table~\ref{tab:compresults_summary1}.

The detailed inspection of Figure~\ref{fig:Runtime} and 
Table~\ref{tab:compresults_summary1} together with Table~\ref{tab:testerror} 
shows that the computational efforts required to solve the benchmark 
instances increase signigicantly more for FD-BLP-Plan+ compared to 
FD-SAT-Plan+ as the learned BNN structure gets more complex (i.e., from 
smallest BNN structure of Inventory, to moderate size BNN structures 
of Navigation and Cellda, to the largest BNN structure of SysAdmin).  
Detailed presentation of the run time results for each instance are 
provided in \ref{app:results}.

\subsection{Planning Performance on the Factored Planning Problems}

Finally, we test the planning efficiency of the incremental factored 
planning algorithm, namely Algorithm~\ref{incremental}, for solving 
the factored planning problem $\Pi$.

\begin{table}[H]
  \centering
  \caption{Summary of the computational results presented in 
  \ref{app:results} including the average runtimes in seconds 
  for both FD-SAT-Plan+ and FD-BLP-Plan+ over all four domains 
  for the factored planning problem within 1 hour time limit.}
  \label{tab:compresults_summary2}
  \begin{tabular}{| l | c | c |}
    \hline
    Domains & FD-SAT-Plan+ & FD-BLP-Plan+ \\ \hline
    Navigation & 529.11 & 1282.82 \\ \hline
    Inventory & 68.88 & 0.66 \\ \hline
    SysAdmin & 2463.79 & 3006.27 \\ \hline
    Cellda & 512.51 & 524.53 \\ \hline
    \hline
    Coverage & 23/27 & 19/27 \\ \hline
    Optimality Proved & 23/27 & 19/27 \\ \hline
  \end{tabular}
\end{table}

In Table~\ref{tab:compresults_summary2}, we present the summary of the 
computational results including the average runtimes in seconds, the 
total number of instances for which a feasible solution is returned 
(i.e., coverage), and the total number of instances for which an optimal 
solution is returned (i.e., optimality proved), for both FD-SAT-Plan+ 
and FD-BLP-Plan+ using Algorithm~\ref{incremental} over all four domains 
for the factored planning problem within 1 hour time limit. The 
analysis of Table~\ref{tab:compresults_summary2} shows that FD-SAT-Plan+ 
with Algorithm~\ref{incremental} covers 23 out of 27 problem instances 
by returning an incumbent solution to the factored planning problem 
compared to FD-BLP-Plan+ with Algorithm~\ref{incremental} which runs out 
of 1 hour time limit in 8 out of 27 instances before finding an incumbent 
solution. Similarly, FD-SAT-Plan+ with Algorithm~\ref{incremental} 
proves the optimality of the solutions found in 23 out of 27 problem 
instances compared to FD-BLP-Plan+ with Algorithm~\ref{incremental} 
which only proves the optimality of 19 out of 27 solutions within 1 hour 
time limit.

Overall, the constraint generation algorithm successfully verified the plans 
found for the factored planning problem $\Pi$ in three out of four domains with 
low computational cost. In contrast, the incremental factored planning 
algorithm spent significantly more time in SysAdmin domain as evident in 
Table~\ref{tab:compresults_summary2}. Over all instances, we observed that at 
most 5 instances required constraint generation to find a plan where the maximum 
number of constraints required was at least 6, namely for (Sys,4,3) instance. 
Detailed presentation of the run time results and the number of generalized 
landmark constraints generated for each instance are provided in \ref{app:results}.

\section{Discussion and Directions for Future Work}

In this section, we discuss the strengths and limitations of our learning 
and planning framework, focusing on the topics of (i) domain discretization, 
(ii) grounded versus lifted model learning, (iii) assumptions on the reward 
function, (iv) linearity assumptions on functions $C$, $I$, $G$, $R$, (v) 
the availability of real-world (or simulated) interaction, and (vi) exploration 
policy choices for data collection.  We also discuss opportunities for future work that can relax many of these assumptions.

\begin{enumerate}[(i)]
    \item Discretization of non-binary domains: One of the main assumptions 
we make in this work is the knowledge of $m$, which denotes the total bits of 
precision used to represent the domains of non-binary state and action variables, 
to learn the known deterministic factored planning problem $\Pi$. In order to 
avoid any assumptions on the value of $m$ in our experiments, we have limited 
our domain instances to only include state and action variables with binary and/or 
bounded integer domains. Our framework can be extended to handle variables with 
bounded continuous domains upto a fixed number of decimals where the value of the 
number of decimals is chosen at training time to learn an accurate $\tilde{T}$. 
An important area of future work here is to improve the accuracy of BNNs through 
the explicit use of the information about the original known domains of state and 
action variables. That is, the current training of BNNs treat each input and output 
unit independently. Therefore, future research that focuses on learning of BNNs 
with non-Boolean domains can further improve the effectiveness of our framework.
    \item Grounded versus lifted model learning: In this paper, we learn grounded 
representations of $\Pi$, that is, we learn a transition function $\tilde{T}$ for 
each instance of the domain. Under this assumption, we only plan in the realized 
(i.e., grounded) instances of the world over which the data is collected. For future 
work, we will investigate methods for learning and planning with lifted representations 
of problem $\Pi$.
    \item Assumptions on the reward function $R$: We have assumed the complete 
knowledge of $R$ since in many planning domains (e.g., Navigation, Inventory etc.), 
the planning objective is user-specified (e.g., in the Navigation domain the agent 
knows where it wants to go, in the Inventory domain the user knows he/she wants to 
minimize total cost etc.). Given this assumption, we have learned a transition 
function $T$ that is independent of $R$. As a result, our learning and planning 
framework is \emph{general} with respect to $R$ (i.e., planning with respect to a 
new $R$ would simply mean the modification of the set of soft clauses (\ref{sat11}) 
or the objective function (\ref{blp7})). If we assumed $R$ can also be measured and 
collected as data, our framework can easily be extended to 
handle planning for an unknown reward function $R$ by also learning $R$ from 
training data.
    \item Linearity assumptions: In order to compile the learned planning problem 
$\tilde{\Pi}$ into WP-MaxSAT and BLP, we have assumed that functions $C$, $I$, $G$, 
$R$ are linear. As a result, our compilation-based planners can be solved 
by state-of-the-art WP-MaxSAT~\cite{Davies2013} and BLP solvers~\cite{IBM2017}. 
This experimental limitation can be lifted by changing our branch-and-solver (i.e., 
for FD-BLP-Plan+) to a more general solver that can handle nonlinearities with 
optimality guarantees (e.g., a spatial branch-and-bound solver such as 
SCIP~\cite{SCIP2017}). For example, the model-based metric hybrid planner 
\textit{SCIPPlan}~\cite{Say2018b,Say2019} leverages the spatial branch-and-bound 
solver of SCIP to plan in domains with nonlinear functions $C$, $I$, $G$, and action 
and state variables with mixed (continuous and discrete) domains.
    \item Availability of real-world (or simulated) interaction: In this work, we 
assumed the presence of historical data that enables model learning but do not always 
assume the availability of online access to real-world (or simulated) feedback. Therefore, 
we have investigated two distinct scenarios for our learning and planning framework, 
namely planning, (i) without and (ii) with, the availability of real-world (or 
simulated) interaction. Scenario (i) assumes the planner must produce a plan without 
having any feedback from the real-world (or simulator). Under scenario (i), all 
information available about transition function $T$ is the learned transition function 
$\tilde{T}$. As a result, we assumed $\tilde{\Pi}$ is an accurate approximation of 
$\Pi$ (as demonstrated experimentally in Table~\ref{tab:testerror}) and planned 
optimally with respect to $\tilde{\Pi}$. Scenario (ii) assumes the limited availability 
of real-world (or potentially expensive simulated) interaction with $\Pi$. Under scenario 
(ii), the planner can correct its plans with respect to the new observations. In this 
work, we do not employ a re-training technique for $T$ since (i) $\Pi$ is static (i.e., 
$\Pi$ does not change between the time training data was collected and 
planning), and (ii) we assume the amount of training data is significantly larger than 
the newly collected data. As a result, we do not re-train the learned transition function 
$T$ but instead correct the planner using the constraint generation framework introduced 
in Section 5.
    \item Data collection: Data acquisition is an important part of machine learning 
which can directly effect the quality of the learned model. In this work, we have 
assumed the availability of data and assumed data as an input parameter to our learning 
and planning framework. In our experiments, we have considered the scenario 
where the learned BNN model can be incorrect (i.e., scenario (ii) as described previously). 
Under scenario (ii), we have investigated how to repair our planner based on its 
interaction with the real-world (or simulator). Clearly, if the model that is learned 
does not include \emph{any} feasible plans $\pi$ to $\Pi$ (i.e., there does not exist 
a plan that is both a solution to $\Pi$ and $\tilde{\Pi}$), then Algorithm~\ref{incremental} 
terminates, as shown in \ref{app:proof}. If this is the case, then it is an interesting 
direction for future work to consider how to collect data in order to accurately learn $T$. 
However, considerations of such methods (e.g., active learning, investigation of different 
exploration policies etc.) are orthogonal to the scope of the technical contributions of 
this paper.
\end{enumerate}

\section{Conclusion}

In this work, we utilized the efficiency and ability of BNNs to 
learn complex state transition models of factored planning domains 
with discretized state and action spaces.  We introduced two novel 
compilations, a WP-MaxSAT (FD-SAT-Plan+) and a BLP (FD-BLP-Plan+) 
encodings, that directly exploit the structure of BNNs to plan for 
the learned factored planning problem, which provide optimality 
guarantees with respect to the learned model when they successfully 
terminate. Theoretically, we have shown that our SAT-based Bi-Directional 
Neuron Activation Encoding is asymptotically the most compact encoding 
in the literature, and is efficient with Unit Propagation (UP), which 
is one of the most important efficiency indicators of a SAT-based encoding. 

We further introduced a finite-time incremental factored 
planning algorithm based on generalized landmark constraints that 
improve planning accuracy of both FD-SAT-Plan+ and FD-BLP-Plan+ through 
simulated or real-world interaction.
Experimentally, we demonstrate the computational efficiency of our 
Bi-Directional Neuron Activation Encoding in comparison to the 
Uni-Directional Neuron Activation Encoding~\cite{Say2018}. Overall, 
our empirical results showed we can accurately learn complex state 
transition models using BNNs and demonstrated strong performance in both 
the learned and original domains. In sum, this work provides two novel 
and efficient factored state and action transition learning and planning 
encodings for BNN-learned transition models, thus providing new and
effective tools for the data-driven model-based planning community.

\section{Acknowledgements}

This work has been funded by the Natural Sciences and Engineering Research Council (NSERC) of Canada.

\appendix
\addcontentsline{toc}{section}{Appendices}
\section*{Appendices}

\section{CNF Encoding of the Cardinality Networks}
\label{app:cnf}

The CNF encoding of Cardinality Networks ($CN_{k}^{=}$) is as follows~\cite{Asin2009}.

\paragraph{Half Merging Networks} Given two inputs of Boolean variables 
$x_1, \dots, x_{n}$ and $y_1, \dots, y_n$, Half Merging (HM) Networks 
merge inputs into a single output of size $2n$ using the CNF encoding as 
follows.

For input size $n=1$:
\begin{align}
&HM(\langle x_1\rangle, \langle y_1 \rangle \rightarrow \langle c_1, c_2\rangle) 
= (\neg x_1 \vee \neg y_1 \vee c_2) 
\wedge (\neg x_1 \vee c_1) \wedge (\neg y_1 \vee c_1)\nonumber\\
&\wedge (x_1 \vee y_1 \vee \neg c_1) \wedge (x_1 \vee \neg c_2) \wedge (y_1 \vee \neg c_2) 
\label{card3}
\end{align}

For input size $n>1$:
\begin{align}
&HM(\langle x_1, \dots, x_n \rangle, \langle y_1, \dots, y_n\rangle \rightarrow 
\langle d_1, c_2,\dots,c_{2n-1}, e_{n}\rangle) = H_{o} \wedge H_{e} \wedge H'\label{card4}\\
&H_{o} = HM(\langle x_1, x_3 \dots, x_{n-1}\rangle, \langle y_1, y_3 \dots, y_{n-1}\rangle 
\rightarrow \langle d_1,\dots,d_{n} \rangle) \label{card5}\\
&H_{e} = HM(\langle x_2, x_4 \dots, x_n \rangle, \langle y_2, y_4 \dots, y_n \rangle 
\rightarrow \langle e_1,\dots,e_{n}\rangle)\label{card6}\\
&H' = \bigwedge^{n-1}_{i=1} (\neg d_{i+1} \vee \neg e_{i} \vee c_{2i+1}) 
\wedge (\neg d_{i+1} \vee c_{2i}) \wedge (\neg e_{i} \vee c_{2i}) \wedge \nonumber\\
&\bigwedge^{n-1}_{i=1} (d_{i+1} \vee e_{i} \vee \neg c_{2i}) 
\wedge (d_{i+1} \vee \neg c_{2i+1}) \wedge (e_{i} \vee \neg c_{2i+1}) \label{card7}
\end{align}

\paragraph{Half Sorting Networks} Given an input of Boolean variables 
$x_1, \dots, x_{2n}$, Half Sorting (HS) Networks sort the variables 
with respect to their value assignment as follows.

For input size $2n=2$:
\begin{align}
&HS(\langle x_1, x_2\rangle \rightarrow \langle c_1,c_2\rangle) = HM(\langle x_1\rangle, \langle x_2\rangle \rightarrow \langle c_1,c_{2}\rangle) \label{card8}
\end{align}

For input size $2n>2$:
\begin{align}
&HS(\langle x_1, \dots, x_{2n}\rangle \rightarrow \langle c_1,\dots,c_{2n} \rangle) = H_{1} \wedge H_{2} \wedge H_{M} \label{card9}\\
&H_{1} = HS(\langle x_1, \dots, x_n\rangle \rightarrow \langle d_1,\dots,d_{n}\rangle) \label{card10}\\
&H_{2} = HS(\langle x_{n+1}, \dots, x_{2n} \rangle \rightarrow \langle {d'}_1,\dots,{d'}_{n}\rangle)) \label{card11}\\
&H_{M} = HM(\langle d_1, \dots, d_n\rangle,\langle {d'}_1, \dots, {d'}_n\rangle \rightarrow \langle c_1,\dots,c_{2n}\rangle) \label{card12}
\end{align}

\paragraph{Simplified Merging Networks} Given two inputs of Boolean variables 
$x_1, \dots, x_{n}$ and $y_1, \dots, y_n$, Simplified Merging (SM) Networks 
merge inputs into a single output of size $2n$ using the CNF encoding as follows.

For input size $n=1$:
\begin{align}
&SM(\langle x_1\rangle, \langle y_1\rangle \rightarrow \langle c_1,c_2\rangle) = HM(\langle x_1\rangle, \langle y_1\rangle \rightarrow \langle c_1, c_2\rangle)
\label{card13}
\end{align}

For input size $n>1$:
\begin{align}
&SM(\langle x_1, \dots, x_n\rangle, \langle y_1, \dots, y_n\rangle 
\rightarrow \langle d_1,c_2,\dots,c_{n+1}\rangle) = S_{o} \wedge S_{e} \wedge S' \label{card14}\\
&S_{o} = SM(\langle x_1, x_3 \dots, x_{n-1}\rangle, \langle y_1, y_3 \dots, y_{n-1}\rangle 
\rightarrow \langle d_1,\dots,d_{\frac{n}{2} + 1}\rangle) \label{card15}\\
&S_{e} = SM(\langle x_2, x_4 \dots, x_n\rangle, \langle y_2, y_4 \dots, y_n\rangle 
\rightarrow \langle e_1,\dots,e_{\frac{n}{2} + 1}\rangle) \label{card16}\\
&S' = \bigwedge^{n/2}_{i=1} (\neg d_{i+1} \vee \neg e_{i} \vee c_{2i+1}) 
\wedge (\neg d_{i+1} \vee c_{2i}) \wedge (\neg e_{i} \vee c_{2i}) \wedge \nonumber\\
&\bigwedge^{n/2}_{i=1} (d_{i+1} \vee e_{i} \vee \neg c_{2i}) 
\wedge (d_{i+1} \vee \neg c_{2i+1}) \wedge (e_{i} \vee \neg c_{2i+1})
\label{card17}
\end{align}

Note that unlike HM, SM counts the number of variables assigned 
to true from input variables $x_1, \dots, x_n$ and $y_1, \dots, y_n$ 
upto $n+1$ bits instead of $2n$.

\paragraph{Cardinality Networks} Given an input of Boolean variables 
$x_1,\dots, x_n$ with $n = k u$ where $p,u \in \mathbb{N}$ and $k$ is the smallest power of 2 such that $k>p$, 
the CNF encoding of Cardinality Networks ($CN_{k}^{=}$) is as follows.

For input size $n = k$:
\begin{align}
&CN_{k}^{=}(\langle x_1,\dots,x_n\rangle \rightarrow \langle c_1,\dots,c_n\rangle) = HS(\langle x_1, \dots, x_n\rangle 
\rightarrow \langle c_1,\dots,c_n\rangle) \label{card18}
\end{align}

For input size $n > k$:
\begin{align}
&CN_{k}^{=}(\langle x_1,\dots,x_n\rangle \rightarrow \langle c_1,\dots,c_{k}\rangle) = C_{1} \wedge C_{2} \wedge C_{M} \label{card19}\\
&C_{1} = CN_k^{=}(\langle x_1,\dots,x_k\rangle \rightarrow \langle d_1,\dots,d_k\rangle) \label{card20}\\
&C_{2} = CN_k^{=}(\langle x_{k+1},\dots,x_n\rangle \rightarrow \langle {d'}_1,\dots,{d'}_k\rangle) \label{card21}\\
&C_{M} = SM(\langle d_1, \dots, d_k\rangle,\langle{d'}_1, \dots, {d'}_{k}\rangle \rightarrow \langle c_1,\dots,c_{k+1}\rangle) \label{card22}
\end{align}

\section{Proof for Algorithm~\ref{incremental}}
\label{app:proof}

Given hard clauses (\ref{sat2}-\ref{sat6}) and Theorem~\ref{proof:gac}, Corollary~\ref{proof:forwardpass} follows.

\begin{corollary}[Forward Pass]
\label{proof:forwardpass}
Given the values of state $\bar{Y}^{i}_{s,t}$ and action $\bar{X}^{i}_{a,t}$ decision variables for all bits $1\leq i \leq m$ and time step $t\in \{1,\dots,H\}$, and the learned transition function $\tilde{T}$, hard clauses (\ref{sat2}-\ref{sat6}) deterministically assign values to all state decision variables $Y^{i}_{s,t+1}$ through Unit Propagation (from Theorem~\ref{proof:gac}) such that $\tilde{T}(\langle \bar{Y}^{i}_{s,t} | s \in S, 1 \leq i \leq m\rangle{}^\frown \langle \bar{X}^{i}_{a,t} | a \in A, 1 \leq i \leq m\rangle) = \langle \bar{Y}^{i}_{s,t+1} | s \in S, 1 \leq i \leq m\rangle$.
\end{corollary}

\begin{theorem}[Finiteness of Algorithm~\ref{incremental}]
\label{proof:alg}
Let $\tilde{\Pi} = \langle S,A,C,\tilde{T},I,G,R \rangle$ be the learned deterministic factored planning problem. For a given horizon $H$ and $m$-bit precision, Algorithm~\ref{incremental} terminates in finite number of iterations $n\leq 2^{|A|\cdot m\cdot H}$.
\end{theorem}

\begin{proof}[Proof by Induction.]
Let $V$ be the set of all value tuples for all action variables $A$ with $m$ bits of precision and time steps $t \in \{1,\dots,H\}$ such that $\langle \bar{A}^{1}, \dots, \bar{A}^{H}\rangle \in V$. From the definition of the decision variable $X^{i}_{a,t}$ in Section 3.1 and the values $\bar{A}^t = \langle \bar{a}_{1}^t, \dots, \bar{a}_{|A|}^t \rangle$ of action variables $A$ for all time steps $t \in \{1,\dots,H\}$, the value of every action decision variable $X^{i}_{a,t}$ is set using the binarization formula such that $\bar{a}^t  = -2^{m-1}\bar{X}^{m}_{a,t} + \sum_{i=1}^{m-1}2^{i-1}\bar{X}^{i}_{a,t}$. Given the initial values $\bar{S}^{i,s}_{I}$, hard clause (\ref{sat1}) sets the values of state decision variables $Y^{i}_{s,1}$ at time step $t=1$ for all bits $1\leq i \leq m$. Given the values of state $\bar{Y}^{i}_{s,1}$ and action $\bar{X}^{i}_{a,1}$ decision variables at time step $t=1$, the values of state decision variables $\bar{Y}^{i}_{s,2}$ are set (from Corollary~\ref{proof:forwardpass}). Similarly, using the values of action decision variables $\bar{X}^{i}_{a,t}$ for time steps $t\in\{2,\dots,H\}$, the values of state decision variables are set $\bar{Y}^{i}_{s,t}$ sequentially for the remaining time steps $t\in\{3,\dots,H+1\}$. Given we have shown that each element $\langle \bar{A}^{1}, \dots, \bar{A}^{H}\rangle \in V$ has a corresponding value tuple for state variables $S$ and time steps $t\in\{2,\dots,H+1\}$, we denote $V'\subseteq V$ as the subset of feasible value tuples for action variables $A$ and state variables $S$ with respect to hard clause (\ref{sat7}) (i.e., global function $C=true$) for time steps $t\in\{1,\dots,H\}$ and hard clause (\ref{sat10}) (i.e., goal state function $G=true$) for time step $t=H+1$.

Base Case ($n=1$): In the first iteration $n=1$, Algorithm~\ref{incremental} either proves infeasibility of $\tilde{\Pi}$ if and only if $V' = \emptyset$, or finds values of action variables $\pi = \langle \bar{A}^1, \dots, \bar{A}^H \rangle$. If the planner returns infeasibility of $\tilde{\Pi}$, Algorithm~\ref{incremental} terminates. Otherwise, values of action variables $\pi$ are sequentially simulated for time steps $t \in \{1, \dots, H\}$ given the initial values of state variables $\bar{S}^{i,s}_{I}$ using state transition function $T$ and checked for its feasibility with respect to (i) global function $C$ and (ii) goal state function $G$. If the domain simulator verifies all the propagated values of state variables as feasible with respect to (i) and (ii), Algorithm~\ref{incremental} terminates and returns $\pi$ as a feasible plan for the deterministic factored planning problem $\Pi$. Otherwise, values of action variables $\pi$ are used to generate a generalized landmark hard clause (or constraint) that is added back to the planner, which only removes $\pi$ from the solution space $V'$ such that $V' \leftarrow V' \setminus \pi$.

Induction Hypothesis ($n<i$): Assume that upto iteration $n<i$, Algorithm~\ref{incremental} removes exactly $n$ unique solutions from the solution space $V'$.

Induction Step ($n=i$): Let $n=i$ be the next iteration of Algorithm~\ref{incremental}. In iteration $n=i$, Algorithm~\ref{incremental} either proves infeasibility of $\tilde{\Pi}$ if and only if $V' = \emptyset$, or finds a value tuple $\pi = \langle \bar{A}^1, \dots, \bar{A}^H \rangle$ of action variables $A$ and time steps $t \in \{1, \dots, H\}$. If the planner returns infeasibility of $\tilde{\Pi}$, Algorithm~\ref{incremental} terminates. Otherwise, values of action variables $\pi$ are sequentially simulated for time steps $t \in \{1, \dots, H\}$ given the initial values of state variables $\bar{S}^{i,s}_{I}$ using state transition function $T$ and checked for its feasibility with respect to (i) global function $C$ and (ii) goal state function $G$. If the domain simulator verifies all the propagated values of state variables as feasible with respect to (i) and (ii), Algorithm~\ref{incremental} terminates and returns $\pi$ as a feasible plan for the deterministic factored planning problem $\Pi$. Otherwise, values of action variables $\pi$ are used to generate a generalized landmark hard clause (or constraint) that is added back to the planner, which only removes $\pi$ from the solution space $V'$ such that $V' \leftarrow V' \setminus \pi$.

By induction in at most $n = |A|\cdot m \cdot H$ iterations, Algorithm~\ref{incremental} either (i) proves there does not exist values $\pi = \langle \bar{A}^1, \dots, \bar{A}^H \rangle$ of action variables $A$ that is both a solution to $\tilde{\Pi}$ and $\Pi$ by reducing $V'$ to an emptyset, or (ii) returns $\pi$ as a solution to the deterministic factored planning problem $\Pi$ (i.e., $|V'| \geq 1$), and terminates.

\end{proof}

\section{Online Repositories}
\label{app:dom}

The respective online repositories for FD-SAT-Plan+ and FD-BLP-Plan+ used to generate experiments in this article are the following:
\begin{itemize}
    \item \url{https://github.com/saybuser/FD-SAT-Plan}\ , and
    \item \url{https://github.com/saybuser/FD-BLP-Plan}\ .
\end{itemize}
Formal text representations of all domains described and experimented in this article can be found in these repositories and read in by the respective planners.


\section{Computational Results}
\label{app:results}

\begin{table}[H]
  \centering
  \caption{Computational results including the runtimes and the total number of generalized landmark constraints generated for both FD-SAT-Plan+ and FD-BLP-Plan+ over all 27 instances within 1 hour time limit. For the instances that time out, secondary results on the solution quality of the returned plans (i.e., their duality gap) are provided.}
  \label{tab:compresults}
  \resizebox{\textwidth}{!}{%
  \begin{tabular}{| l | c | c | c | c | c | c |}
    \hline
     \multicolumn{3}{| c |}{Non-Incremental Runtimes} & \multicolumn{2}{| c |}{Incremental Runtimes} & \multicolumn{2}{| c |}{No. of Generalized Landmarks}\\ \hline
    Instances & FD-SAT-Plan+ & FD-BLP-Plan+ & FD-SAT-Plan+ & FD-BLP-Plan+ & FD-SAT-Plan+ & FD-BLP-Plan+ \\ \hline
    Nav,3x3,4 & 3.15 & 1.41 & 3.15 & 1.41 & 0 & 0 \\ \hline
    Nav,3x3,5 & 5.55 & 3.78 & 5.55 & 3.78 & 0 & 0 \\ \hline
    Nav,3x3,6 & 9.19 & 17.82 & 9.19 & 17.82 & 0 & 0 \\ \hline
    Nav,4x4,5 & 82.99 & 107.59 & 82.99 & 107.59 & 0 & 0 \\ \hline
    Nav,4x4,6 & 190.77 & 159.42 & 190.77 & 159.42 & 0 & 0 \\ \hline
    Nav,4x4,7 & 303.05 & 455.32 & 303.05 & 455.32 & 0 & 0 \\ \hline
    Nav,5x5,8 & 1275.36 & 3600$<$,no sol. & 1275.36 & 3600$<$,no sol. & 0 & 0 \\ \hline
    Nav,5x5,9 & 753.27 & 3600$<$,no sol. & 753.27 & 3600$<$,no sol. & 0 & 0 \\ \hline
    Nav,5x5,10 & 2138.62 & 3600$<$,no sol. & 2138.62 & 3600$<$,no sol. & 0 & 0 \\ \hline
    Inv,2,5 & 26.92 & 0.37 & 26.92 & 0.37 & 0 & 0 \\ \hline
    Inv,2,6 & 33.25 & 0.45 & 33.25 & 0.45 & 0 & 0 \\ \hline
    Inv,2,7 & 40.15 & 0.56 & 40.15 & 0.56 & 0 & 0 \\ \hline
    Inv,4,6 & 63.18 & 0.51 & 63.18 & 0.51 & 0 & 0 \\ \hline
    Inv,4,7 & 79.19 & 0.59 & 79.19 & 0.59 & 0 & 0 \\ \hline
    Inv,4,8 & 86.57 & 0.74 & 170.56 & 1.49 & 1 & 1 \\ \hline
    Sys,4,2 & 24.19 & 37.63 & 24.19 & 37.63 & 0 & 0 \\ \hline
    Sys,4,3 & 619.57 & 3600$<$,100\% & 3600$<$,no sol. & 3600$<$,no sol. & 6$\leq$ & n/a \\ \hline
    Sys,4,4 & 1561.78 & 3600$<$,no sol. & 3600$<$,no sol. & 3600$<$,no sol. & 3$\leq$ & n/a \\ \hline
    Sys,5,2 & 358.53 & 3600$<$,no sol. & 358.53 & 3600$<$,no sol. & 0 & n/a \\ \hline
    Sys,5,3 & 3600$<$,75\% & 3600$<$,no sol. & 3600$<$,no sol. & 3600$<$,no sol. & n/a & n/a \\ \hline
    Sys,5,4 & 3600$<$,100\% & 3600$<$,no sol. & 3600$<$,no sol. & 3600$<$,no sol. & n/a & n/a \\ \hline
    Cellda,x,10 & 197.16 & 106.93 & 592.44 & 405.52 & 2 & 2 \\ \hline
    Cellda,x,11 & 219.72 & 403.24 & 835.36 & 1539.12 & 2 & 3 \\ \hline
    Cellda,x,12 & 522.15 & 527.56 & 522.15 & 527.56 & 0 & 0 \\ \hline
    Cellda,y,8 & 144.95 & 89.64 & 144.95 & 89.64 & 0 & 0 \\ \hline
    Cellda,y,9 & 404.39 & 40.9 & 404.39 & 40.9 & 0 & 0 \\ \hline
    Cellda,y,10 & 575.78 & 544.45 & 575.78 & 544.45 & 0 & 0 \\ \hline
    \hline
    Coverage & 27/27 & 20/27 & 23/27 & 19/27 & & \\ \hline
    Opt. Proved & 25/27 & 19/27 & 23/27 & 19/27 & & \\ \hline
  \end{tabular}
  }
\end{table}


\bibliography{bibfile}

\end{document}